\documentclass{article}

\usepackage{microtype}
\usepackage{graphicx}
\graphicspath{ {images/} }
\usepackage{booktabs} 

\usepackage{hyperref}

\usepackage[accepted]{icml2018}

\usepackage{nicefrac}       
\usepackage{amsmath,amssymb}
\usepackage{mathtools}
\usepackage{pifont}
\usepackage[makeroom]{cancel}
\usepackage{placeins}
\usepackage{subcaption}

\usepackage{amsthm}
\theoremstyle{remark}

\theoremstyle{definition}
\newtheorem{definition}{Definition}[section]

\usepackage{thmtools, thm-restate}
\declaretheorem[numberwithin=section]{thm}

\declaretheorem[sibling=thm]{assumption}

\usepackage{xspace}
\DeclareRobustCommand{\eg}{e.g.,\@\xspace}
\DeclareRobustCommand{\ie}{i.e.,\@\xspace}
\DeclareRobustCommand{\wrt}{w.r.t.\@\xspace}
\usepackage[usenames,dvipsnames]{xcolor}
\usepackage{tikz,pgfplots}
\pgfplotsset{compat=newest}
\usepackage[colorinlistoftodos, textwidth=20mm]{todonotes}
\definecolor{citrine}{rgb}{0.89, 0.82, 0.04}
\definecolor{blued}{RGB}{70,197,221}

\newcommand{\realspace}{\mathbb R}      
\newcommand{\transpose}[1]{{#1}^\texttt{T}}
\DeclareMathOperator*{\argmax}{arg\,max}
\DeclareMathOperator*{\argmin}{arg\,min}
\DeclareMathOperator*{\EV}{\mathbb{E}}
\DeclareMathOperator{\Tr}{Tr}
\DeclareMathOperator*{\Cov}{\mathbb{C}ov}
\DeclareMathOperator*{\Var}{\mathbb{V}ar}
\newcommand{\EVV}[2][\ppvect \in \ppspace]{\EV_{#1}\left[{#2}\right]}
\newcommand{\norm}[2][\infty]{\left\|#2\right\|_{#1}}
\newcommand{\Dij}[2]{\frac{\partial^{2}{#1}}{\partial{#2}_i\partial{#2}_j}}
\newcommand{\de}{\,\mathrm{d}}
\newcommand{\dotprod}[2]{\left\langle#1,#2\right\rangle}
\newcommand{\dnabla}{\nabla\!\!\!\!\nabla}
\newcommand{\vtheta}{\boldsymbol{\theta}}
\newcommand{\Aspace}{\mathcal{A}}
\newcommand{\Sspace}{\mathcal{S}}
\newcommand{\Tspace}{\mathcal{T}}
\newcommand{\Transition}{\mathcal{P}}
\newcommand{\Reward}{\mathcal{R}}

\newcommand{\pol}{\pi_{\vtheta}}

\newcommand{\score}[2]{\nabla\log p_{#1}(#2)}

\newcommand{\vTheta}{\boldsymbol{\Theta}}
\newcommand{\vphi}{\boldsymbol{\phi}}
\newcommand{\gradJ}[1]{\nabla J(#1)}
\newcommand{\gradApp}[2]{\widehat{\nabla}_{#2}J(#1)}

\newcommand{\Ets}[2][t]{\mathbb{E}_{#1\vert s}\left[#2\right]}
\newcommand{\Es}[1]{\mathbb{E}_{s}\left[#1\right]}
\newcommand{\Covts}[3][t]{{\mathbb{C}\text{ov}}_{#1\vert s}\left(#2,#3\right)}
\newcommand{\Covs}[2]{{\mathbb{C}\text{ov}}_{s}\left(#1,#2\right)}
\newcommand{\Varts}[2][t]{{\mathbb{V}\text{ar}}_{#1\vert s}\left[#2\right]}
\newcommand{\Vars}[1]{{\mathbb{V}\text{ar}}_{s}\left[#1\right]}
\newcommand{\gradBlack}[1]{\blacktriangledown J(#1)}
\newcommand{\gradIdeal}[1]{\dnabla J(#1)}
\newcommand{\VARRF}{V}
\newcommand{\GRADLOG}{G}
\newcommand{\VARIS}{W}
\newcommand{\HESSLOG}{F}
\newcommand{\wt}[1]{\widetilde{#1}}
\newcommand{\wh}[1]{\widehat{#1}}

\makeatletter
\g@addto@macro\normalsize{%
}
\makeatother

\icmltitlerunning{Stochastic Variance-Reduced Policy Gradient}

\begin{document}

\twocolumn[
\icmltitle{Stochastic Variance-Reduced Policy Gradient}



\icmlsetsymbol{equal}{*}

\begin{icmlauthorlist}
\icmlauthor{Matteo Papini}{equal,polimi}
\icmlauthor{Damiano Binaghi}{equal,polimi}
\icmlauthor{Giuseppe Canonaco}{equal,polimi}
\icmlauthor{Matteo Pirotta}{inria}
\icmlauthor{Marcello Restelli}{polimi}
\end{icmlauthorlist}

\icmlaffiliation{polimi}{Politecnico di Milano, Milano, Italy}
\icmlaffiliation{inria}{Inria, Lille, France}

\icmlcorrespondingauthor{Matteo Papini}{matteo.papini@polimi.it}

\icmlkeywords{Machine Learning, ICML}

\vskip 0.3in
]



\printAffiliationsAndNotice{\icmlEqualContribution} 

\begin{abstract}
        In this paper, we propose a novel reinforcement-learning algorithm consisting in a stochastic variance-reduced version of policy gradient for solving Markov Decision Processes (MDPs).
        Stochastic variance-reduced gradient (SVRG) methods have proven to be very successful in supervised learning.
        However, their adaptation to policy gradient is not straightforward and needs to account for I) a non-concave objective function; II) approximations in the full gradient computation; and III) a non-stationary sampling process.
        The result is SVRPG, a stochastic variance-reduced policy gradient algorithm that leverages on importance weights to preserve the unbiasedness of the gradient estimate.
        Under standard assumptions on the MDP, we provide convergence guarantees for SVRPG with a convergence rate that is linear under increasing batch sizes.
        Finally, we suggest practical variants of SVRPG, and we empirically evaluate them on continuous MDPs.
\end{abstract}

\vspace{-0.05in}
\section{Introduction}
\vspace{-0.05in}
On a very general level, artificial intelligence addresses the problem of an agent that must select the right actions to solve a task. The approach of Reinforcement Learning (RL)~\citep{sutton1998reinforcement} is to learn the best actions by direct interaction with the environment and evaluation of the performance in the form of a reward signal. This makes RL fundamentally different from Supervised Learning (SL), where correct actions are explicitly prescribed by a human teacher (\eg for classification, in the form of class labels). However, the two approaches share many challenges and tools. The problem of estimating a model from samples, which is at the core of SL, is equally fundamental in RL, whether we choose to model the environment, a value function, or directly a policy defining the agent's behaviour. Furthermore, when the tasks are characterized by large or continuous state-action spaces, RL needs the powerful function approximators (\eg neural networks) that are the main subject of study of SL.
In a typical SL setting, a performance function $J(\vtheta)$ has to be optimized \wrt to model parameters $\vtheta$. The set of data that are available for training is often a subset of all the cases of interest, which may even be infinite, leading to optimization of finite sums that approximate the expected performance over an unknown data distribution. When generalization to the complete dataset is not taken into consideration, we talk about Empirical Risk Minimization (ERM). Even in this case, stochastic optimization is often used for reasons of efficiency. The idea of stochastic gradient (SG) ascent \cite{nesterov2013introductory} is to iteratively focus on a random subset of the available data to obtain an approximate improvement direction. At the level of the single iteration, this can be much less expensive than taking into account all the data. However, the sub-sampling of data is a source of variance that can potentially compromise convergence, so that per-iteration efficiency and convergence rate must be traded off with proper handling of meta-parameters.
Variance-reduced gradient algorithms such as SAG \cite{roux2012stochastic}, SVRG \cite{johnson2013accelerating} and SAGA \cite{defazio2014saga} offer better ways of solving this trade-off, with significant results both in theory and practice. Although designed explicitly for ERM, these algorithms address a problem that affects more general machine learning problems. 

In RL, stochastic optimization is rarely a matter of choice, since data must be actively sampled by interacting with an initially unknown environment. In this scenario, limiting the variance of the estimates is a necessity that cannot be avoided, which makes variance-reduced algorithms very interesting.
Among RL approaches, policy gradient \cite{sutton2000policy} is the one that bears the closest similarity to SL solutions. The fundamental principle of these methods is to optimize a parametric policy through stochastic gradient ascent. Compared to other applications of SG, the cost of collecting samples can be very high since it requires to interact with the environment. This makes SVRG-like methods potentially much more efficient than, \eg batch learning.
Unfortunately, RL has a series of difficulties that are not present in ERM. First, in SL the objective can often be designed to be strongly concave (we aim to maximize). This is not the case for RL, so we have to deal with non-concave objective functions. Then, as mentioned before, the dataset is not initially available and may even be infinite, which makes approximations unavoidable. This rules out SAG and SAGA because of their storage requirements, which leaves SVRG as the most promising choice. Finally, the distribution used to sample data is not under direct control of the algorithm designer, but it is a function of policy parameters that change over time as the policy is optimized, which is a form of non-stationarity.
SVRG has been used in RL as an efficient technique for optimizing the per-iteration problem in Trust-Region Policy Optimization~\citep{xu2017svrgtrpo} or for policy evaluation~\citep{du2017svrgpe}.
In both the cases, the optimization problems faced resemble the SL scenario and are not affected by all the previously mentioned issues.

After providing background on policy gradient and SVRG in Section \ref{sec:pre}, we propose SVRPG, a variant of SVRG for the policy gradient framework, addressing all the difficulties mentioned above (see Section \ref{sec:alg}). In Section \ref{sec:conv} we provide convergence guarantees for our algorithm, and we show a convergence rate that has an $O(\nicefrac{1}{T})$ dependence on the number $T$ of iterations. In Section \ref{sec:stopping} we suggest how to set the meta-parameters of SVRPG, while in Section \ref{sec:prac} we discuss some practical variants of the algorithm. Finally, in Section \ref{sec:exp} we empirically evaluate the performance of our method on popular continuous RL tasks.

\vspace{-0.05in}
\section{Preliminaries}\label{sec:pre}
\vspace{-0.05in}
In this section, we provide the essential background on policy gradient methods and stochastic variance-reduced gradient methods for finite-sum optimization.

\vspace{-0.05in}
\subsection{Policy Gradient}\label{subsec:PolicyGradient}
\vspace{-0.05in}
A Reinforcement Learning task~\citep{sutton1998reinforcement} can be modelled with a discrete-time continuous Markov Decision Process (MDP) $M = \{\Sspace,\Aspace,\Transition,\Reward,\gamma,\rho\}$, where $\Sspace$ is a continuous state space; $\Aspace$ is a continuous action space; $\Transition$ is a Markovian transition model, where $\Transition(s'|s,a)$ defines the transition density from state $s$ to $s'$ under action $a$; $\Reward$ is the reward function, where $\Reward(s,a) \in [-R,R]$ is the expected reward for state-action pair $(s,a)$;
$\gamma\in[0,1)$ is the discount factor; and $\rho$ is the initial state distribution.
The agent's behaviour is modelled as a policy $\pi$, where $\pi(\cdot|s)$ is the density distribution over $\Aspace$ in state $s$.
We consider episodic MDPs with effective horizon $H$.\footnote{The episode duration is a random variable, but the optimal policy can reach the target state (\ie absorbing state) in less than $H$ steps. This has not to be confused with a finite horizon problem where the optimal policy is non-stationary.} In this setting, we can limit our attention to trajectories of length $H$. A trajectory $\tau$ is a sequence of states and actions $(s_0,a_0,s_1,a_1,\dots,s_{H-1},a_{H-1})$ observed by following a stationary policy, where $s_0 \sim \rho$.
We denote with $p(\tau|\pi)$ the density distribution induced by policy $\pi$ on the set $\Tspace$ of all possible trajectories (see Appendix~\ref{A:gradient_estimators} for the definition), and with $\Reward(\tau)$ the total discounted reward provided by trajectory $\tau$:
$\Reward(\tau) = \sum_{t=0}^{H-1}\gamma^t\Reward(s_t,a_t).$
Policies can be ranked based on their expected total reward: $J(\pi) = \EVV[\tau \sim p(\cdot|\pi)]{\Reward(\tau)|M}$.
Solving an MDP $M$ means finding $\pi^* \in \argmax_{\pi} \{J(\pi)\}$.

Policy gradient methods restrict the search for the best performing policy over a class of parametrized policies $\Pi_{\vtheta}=\{\pol: \vtheta \in \realspace^d\}$, with the only constraint that $\pol$ is differentiable \wrt $\vtheta$. For sake of brevity, we will denote the performance of a parametric policy with $J(\vtheta)$ and the probability of a trajectory $\tau$ with $p(\tau|\vtheta)$ (in some occasions, $p(\tau|\vtheta)$ will be replaced by $p_{\vtheta}(\tau)$ for the sake of readability).
The search for a locally optimal policy is performed through gradient ascent, where the policy gradient
is \cite{sutton2000policy, Peters2008reinf}:
\begin{align} \label{E:policygradient}
        \gradJ{\vtheta} = \EVV[\tau \sim p(\cdot|\vtheta)]{\score{\vtheta}{\tau}\Reward(\tau)}.
\end{align}
Notice that the distribution defining the gradient is induced by the current policy. This aspect introduces a nonstationarity in the sampling process. Since the underlying distribution changes over time, it is necessary to resample at each update or use weighting techniques such as importance sampling.
Here, we consider the \emph{online learning scenario}, where trajectories are sampled by interacting with the environment at each policy change. 
In this setting, stochastic gradient ascent is typically employed.
At each iteration $k >0$, a batch $\mathcal{D}_N^k = \{\tau_i\}_{i=0}^N$ of $N>0$ trajectories is collected using policy $\pi_{\vtheta_k}$.
The policy is then updated as $\vtheta_{k+1}  = \vtheta_k + \alpha\gradApp{\vtheta_k}{N}$, where $\alpha$ is a step size and $\gradApp{\vtheta}{N}$ is an estimate of Eq.~\eqref{E:policygradient} using $\mathcal{D}_N^k$. The most common policy gradient estimators (\eg REINFORCE~\citep{williams1992simple} and G(PO)MDP~\citep{baxter2001infinite}) can be expressed as follows
\begin{align} \label{E:policygradient.estimate}
        \gradApp{\vtheta}{N} = \frac{1}{N}\sum_{n=1}^{N} g(\tau_i|\vtheta), \quad \tau_i \in \mathcal{D}_N^k,
\end{align}
where $g(\tau_i|\vtheta)$ is an estimate of $\score{\vtheta}{\tau_i}\Reward(\tau_i)$.
Although the REINFORCE definition is simpler than the G(PO)MDP one, the latter is usually preferred due to its lower variance.
We refer the reader to Appendix~\ref{A:gradient_estimators} for details and a formal definition of $g$.

The main limitation of plain policy gradient is the high variance of these estimators.
The na\"ive approach of increasing the batch size is not an option in RL due to the high cost of collecting samples, \ie by interacting with the environment.
For this reason, literature has focused on the introduction of baselines (\ie functions $b : \mathcal{S} \times \mathcal{A} \to \realspace$) aiming to reduce the variance~\citep[\eg][]{williams1992simple,Peters2008reinf,Thomas2017actionbaseline,wu2018variance}, see Appendix~\ref{A:gradient_estimators} for a formal definition of $b$.
These baselines are usually designed to minimize the variance of the gradient estimate, but even them need to be estimated from data, partially reducing their effectiveness.
On the other hand, there has been a surge of recent interest in variance reduction techniques for gradient optimization in supervised learning (SL).
Although these techniques have been mainly derived for finite-sum problems, we will show in Section~\ref{sec:alg} how they can be used in RL.
In particular, we will show that the proposed SVRPG algorithm can take the best of both worlds (\ie SL and RL) since it can be plugged into a policy gradient estimate using baselines.
The next section has the aim to describe variance reduction techniques for finite-sum problems. In particular, we will present the SVRG algorithm that is at the core of this work.

\vspace{-0.05in}
\subsection{Stochastic Variance-Reduced Gradient}\label{sec:svrg}
\vspace{-0.05in}
Finite-sum optimization is the problem of maximizing an objective function $f(\vtheta)$ which can be decomposed into the sum or average of a finite number of functions $z_i(\vtheta)$:
\begin{align*}
        \max_{\vtheta} \left\{ f(\vtheta) = \frac{1}{N}\sum_{i=1}^{N}z_i(\vtheta)\right\}.
\end{align*}
This kind of optimization is very common in machine learning, where each $z_i$ may correspond to a data sample $x_i$ from a dataset $\mathcal{D}_N$ of size $N$ (\ie $z_i(\vtheta) = z(x_i|\vtheta)$). 
A common requirement is that $z$ must be smooth and concave in $\vtheta$.\footnote{Note that we are considering a maximization problem instead of the classical minimization one.} 
Under this hypothesis, full gradient (FG) ascent~\citep{cauchy1847methode} with a constant step size achieves a linear convergence rate in the number $T$ of iterations (\ie parameter updates)~\citep{nesterov2013introductory}.
However, each iteration requires $N$ gradient computations, which can be too expensive for large values of $N$. Stochastic Gradient (SG) ascent~\citep[\eg][]{robbins1951stochastic,bottou2004large} overcomes this problem by sampling a single sample $x_i$ per iteration, but a vanishing step size is required to control the variance introduced by sampling. As a consequence, the lower per-iteration cost is paid with a worse, sub-linear convergence rate \cite{nemirovskii1983problem}.
Starting from SAG, a series of variations to SG have been proposed to achieve a better trade-off between convergence speed and cost per iteration: \eg SAG~\citep{roux2012stochastic}, SVRG~\cite{johnson2013accelerating}, SAGA~\cite{defazio2014saga}, Finito~\cite{defazio2014finito}, and MISO~\cite{mairal2015incremental}. 
The common idea is to reuse past gradient computations to reduce the variance of the current estimate.
In particular, Stochastic Variance-Reduced Gradient (SVRG) is often preferred to other similar methods for its limited storage requirements, which is a significant advantage when deep and/or wide neural networks are employed.  

The idea of SVRG (Algorithm~\ref{alg:svrg}) is to alternate full and stochastic gradient updates. 
Each $m = O(N)$ iterations, a snapshot $\widetilde{\vtheta}$ of the current parameter is saved together with its full gradient $\nabla f(\widetilde{\vtheta}) = \frac{1}{N} \sum_i \nabla z(x_i|\widetilde{\vtheta})$.
Between snapshots, the parameter is updated with $\blacktriangledown f(\vtheta)$, a gradient estimate corrected using stochastic gradient. For any $t \in \{0,\ldots,m-1\}$:
\begin{equation}\label{E:svrg.gradient.correction}
        \blacktriangledown f(\vtheta_{t}) := v_t = \nabla f(\wt{\vtheta}) + \nabla z(x | \vtheta_t) - \nabla z(x | \wt{\vtheta}),
\end{equation} 
where $x$ is sampled uniformly at random from $\mathcal{D}_N$ (\ie $x \sim \mathcal{U}(\mathcal{D}_N)$).
Note that $t=0$ corresponds to a FG step (\ie $\blacktriangledown f(\vtheta_0) = \nabla f(\wt{\vtheta})$) since $\vtheta_0 := \wt{\vtheta}$.
The corrected gradient $\blacktriangledown f(\vtheta)$ is an unbiased estimate of $\nabla f(\vtheta)$, and it is able to control the variance introduced by sampling even with a fixed step size, achieving a linear convergence rate without resorting to a plain full gradient.

More recently, some extensions of variance reduction algorithms to the non-concave objectives have been proposed~\citep[\eg][]{allen2016variance,reddi2016stochastic,reddi2016fast}. In this scenario, $f$ is typically required to be $L$-smooth, \ie $\norm[2]{\nabla f(\vtheta') - \nabla f(\vtheta)} \leq L\norm[2]{\vtheta'-\vtheta}$ for each $\vtheta,\vtheta'\in\realspace^n$ and for some Lipschitz constant $L$. Under this hypothesis, the convergence rate of SG is $O(\nicefrac{1}{\sqrt{T}})$ \cite{ghadimi2013stochastic}, \ie $T=O(\nicefrac{1}{\epsilon^2})$ iterations are required to get $\norm[2]{\nabla f(\vtheta)}^2\leq\epsilon$. Again, SVRG achieves the same rate as FG \cite{reddi2016stochastic}, which is $O(\frac{1}{T})$ in this case \cite{nesterov2013introductory}. The only additional requirement is to select $\vtheta^*$ uniformly at random among all the $\vtheta_k$ instead of simply setting it to the final value ($k$ being the iterations).

\begin{algorithm}[tb]
	\caption{SVRG}
	\label{alg:svrg}
	\begin{algorithmic}
            \STATE {\bfseries Input:} a dataset $\mathcal{D}_N$, number of epochs $S$, epoch size $m$, step size $\alpha$, initial parameter $\vtheta_{m}^0:=\wt{\vtheta}^0$
		\FOR{$s=0$ {\bfseries to} $S-1$}
        \STATE $\vtheta_0^{s+1} := \wt{\vtheta}^{s} = \vtheta_{m}^s$
        \STATE $\wt{\mu} = \nabla f(\wt{\vtheta}^s)$
		\FOR{$t=0$ {\bfseries to} $m-1$}
        \STATE $x \sim \mathcal{U}\left(\mathcal{D}_N\right)$
		\STATE $v^{s+1}_t = 
			\wt{\mu} + 
			\nabla z(x|\vtheta_t^{s+1}) -
			\nabla z(x|\wt{\vtheta}^{s})
		$
        \STATE $\vtheta_{t+1}^{s+1} = \vtheta_t^{s+1} + \alpha v^{s+1}_t$
		\ENDFOR
		\ENDFOR
        \STATE \underline{Concave case:} \textbf{return} $\vtheta_{m}^S$
        \STATE \underline{Non-Concave case:} \textbf{return} $\vtheta_t^{s+1}$ with $(s,t)$ picked uniformly at random from $\{[0,S-1]\times[0,m-1]\}$
	\end{algorithmic}
\end{algorithm}

\vspace{-0.05in}
\section{SVRG in Reinforcement Learning}\label{sec:alg}
\vspace{-0.05in}
In online RL problems, the usual approach is to tune the batch size of SG to find the optimal trade-off between variance and speed.
Recall that, compared to SL, the samples are not fixed in advance but we need to collect them at each policy change.
Since this operation may be costly, we would like to minimize the number of interactions with the environment.
For these reasons, we would like to apply SVRG to RL problems in order to limit the variance introduced by sampling trajectories, which would ultimately lead to faster convergence.
However, a direct application of SVRG to RL is not possible due to the following issues:
\begin{description}
        \item[Non-concavity:] the objective function $J(\vtheta)$ is typically non-concave.
        \item[Infinite dataset:] the RL optimization cannot be expressed as a finite-sum problem. The objective function is an expected value over the trajectory density $p_{\vtheta}(\tau)$ of the total discounted reward, for which we would need an infinite dataset.
        \item[Non-stationarity:] the distribution of the samples changes over time. In particular, the value of the policy parameter $\vtheta$ influences the sampling process.
\end{description}
To deal with non-concavity, we require $J(\vtheta)$ to be $L$-smooth, which is a reasonable assumption for common policy classes such as Gaussian\footnote{See Appendix~\ref{app:gauss} for more details on the Gaussian policy case.} and softmax~\citep[\eg][]{Furmston2012unifying,pirotta2015lipschitz}.
Because of the infinite dataset, we can only rely on an estimate of the full gradient.
\citet{harikandeh2015stopwasting} analysed this scenario under the assumptions of $z$ being concave, showing that SVRG is robust to an inexact computation of the full gradient. In particular, it is still possible to recover the original convergence rate if the error decreases at an appropriate rate. \citet{bietti2017stochastic} performed a similar analysis on MISO. In Section \ref{sec:conv} we will show how the estimation accuracy impacts on the convergence results with a non-concave objective.
Finally, the non-stationarity of the optimization problem introduces a bias into the SVRG estimator in Eq.~\eqref{E:svrg.gradient.correction}.
To overcome this limitation we employ importance weighting~\citep[\eg][]{rubinstein1981simulation,precup2000eligibility} to correct the distribution shift.

\begin{algorithm}[tb]
	\caption{SVRPG}
	\label{alg:svrpg}
	\begin{algorithmic}
		\STATE {\bfseries Input:} number of epochs $S$, epoch size $m$, step size $\alpha$, batch size $N$, mini-batch size $B$, gradient estimator $g$, initial parameter $\vtheta_{m}^0 \coloneqq \wt{\vtheta}^0 \coloneqq \vtheta_0$
		\FOR{$s=0$ {\bfseries to} $S-1$}
		\STATE $\vtheta_0^{s+1} := \wt{\vtheta}^{s} = \vtheta_{m}^s$
		\STATE Sample $N$ trajectories $\{\tau_j\}$ from $p(\cdot\vert\wt{\vtheta}^{s})$
		\STATE $ \wt{\mu} = \gradApp{\wt{\vtheta}^{s}}{N}$ (see Eq.~\eqref{E:policygradient.estimate})
		\FOR{$t=0$ {\bfseries to} $m-1$}
		\STATE Sample $B$ trajectories $\{\tau_i\}$ from $p(\cdot\vert\vtheta_t^{s+1})$
		\STATE $c^{s+1}_t = \frac{1}{B} \sum\limits_{i=0}^{B-1}
		\begin{aligned}[t]
		\Big( & g(\tau_i|\vtheta_t^{s+1})\\ 
		& \quad{} - \omega(\tau_i|\vtheta^{s+1}_t, \wt{\vtheta}^s) g(\tau_i| \wt{\vtheta}^s) \Big)
		\end{aligned}$
		\STATE $v^{s+1}_t = \wt{\mu} + c^{s+1}_t$ 
		\STATE $\vtheta_{t+1}^{s+1} = \vtheta_t^{s+1} + \alpha v^{s+1}_t$
		\ENDFOR
		\ENDFOR
		\STATE {\bfseries return} $\vtheta_A\coloneqq\vtheta_t^{s+1}$ with $(s,t)$ picked uniformly at random from $\{[0,S-1]\times[0,m-1]\}$
	\end{algorithmic}
\end{algorithm} 

We can now introduce Stochastic Variance-Reduced Policy Gradient (SVRPG) for a generic policy gradient estimator $g$. Pseudo-code is provided in Algorithm \ref{alg:svrpg}.
The overall structure is the same as Algorithm \ref{alg:svrg}, 
but the snapshot gradient is not exact and the gradient estimate used between snapshots is corrected using importance weighting:\footnote{Note that $g$ can be any unbiased estimator, with or without baseline. The unbiasedness is required for theoretical results (\eg Appendix~\ref{A:gradient_estimators}).}
\begin{align*}
        \blacktriangledown J(\vtheta_{t}) &= \wh{\nabla}_N J(\wt{\vtheta}) + g(\tau|\vtheta_t) - \omega(\tau|\vtheta_t, \wt{\vtheta}) g(\tau|\wt{\vtheta})
\end{align*}
for any $t \in \{0,\ldots,m-1\}$,
where $\wh{\nabla}_N J(\wt{\vtheta})$ is as in Eq.~\eqref{E:policygradient.estimate} where $\mathcal{D}_N$ is sampled using the snapshot policy $\pi_{\wt{\vtheta}}$, $\tau$ is sampled from the current policy $\pi_{\vtheta_t}$, and $\omega(\tau|\vtheta_t, \wt{\vtheta}) = \frac{p(\tau|\wt{\vtheta})}{p(\tau|\vtheta_t)}$ is an importance weight from $\pi_{\vtheta_t}$ to the snapshot policy $\pi_{\wt{\vtheta}}$. 
Similarly to SVRG, we have that $\vtheta_0 := \wt{\vtheta}$, and the update is a FG step.
Our update is still fundamentally on-policy since the weighting concerns only the correction term. However, this partial ``off-policyness'' represents an additional source of variance. This is a well-known issue of importance sampling~\citep[\eg][]{thomas2015high}. To mitigate it, we use mini-batches of trajectories of size $B \ll N$ to average the correction, \ie
\begin{align}\label{E:svrpg.estimate.batch}
        \blacktriangledown J(\vtheta_{t}) &:= v_t= \wh{\nabla}_N J(\wt{\vtheta})\\ \notag
                                            & \quad{} + 
        \underbracket{
        \frac{1}{B} \sum_{i=0}^{B-1} \left[
        g(\tau_i|\vtheta_t) - \omega(\tau_i|\vtheta_t, \wt{\vtheta}) g(\tau_i|\wt{\vtheta})
\right]}_{c_t}.
\end{align}
It is worth noting that the full gradient and the correction term have the same expected value: $\EVV[\tau_i \sim p(\cdot|\vtheta_t)]{\frac{1}{B} \sum_{i=0}^{B-1} \omega(\tau_i|\vtheta_t, \wt{\vtheta}) g(\tau_i|\wt{\vtheta})} = \nabla J(\wt{\vtheta})$.\footnote{The reader can refer to Appendix~\ref{A:gradient_estimators} for off-policy gradients and variants of REINFORCE and G(PO)MDP.}
This property will be used to prove Lemma~\ref{L:svrpg.properties}.
The use of mini-batches is also common practice in SVRG since it can yield a performance improvement even in the supervised case~\citep{harikandeh2015stopwasting,konevcny2016mini}. It is easy to show that the SVRPG estimator has the following, desirable properties:

\begin{restatable}[]{lemma}{svrpgprop}\label{L:svrpg.properties}
        Let $\wh{\nabla}_N J(\vtheta)$ be an unbiased estimator of~\eqref{E:policygradient}
and let $\vtheta^* \in \argmin_{\vtheta} \{J(\vtheta)\}$. Then, the SVRG estimate in~\eqref{E:svrpg.estimate.batch} is \emph{unbiased}
\begin{equation}\label{eq:unbiased}
\mathop{\mathbb{E}}
\left[\blacktriangledown J(\vtheta)\right] = \gradJ{\vtheta}.
\end{equation}
and regardless of the mini-batch size $B$:\footnote{
For any vector $\mathbf{x}$, we use $\Var[\mathbf{x}]$ to denote the trace of the covariance matrix, \ie $\Tr(\EVV[]{(\mathbf{x}-\EVV[]{\mathbf{x}})(\mathbf{x}-\EVV[]{\mathbf{x}})^T})$.}
\begin{equation}\label{eq:zerovar}
	\Var\left[\gradBlack{\vtheta^*}\right] = 
    \Var\left[\wh{\nabla}_N J(\vtheta^*)\right].
\end{equation}
\end{restatable}
Previous results hold for both REINFORCE and G(PO)MDP.
In particular, the latter result suggests that an SVRG-like algorithm using $\gradBlack{\vtheta}$ can achieve faster convergence, by performing much more parameter updates with the same data without introducing additional variance (at least asymptotically).
Note that the randomized return value of Algorithm \ref{alg:svrpg} does not affect online learning at all, but will be used as a theoretical tool in the next section.

\vspace{-0.05in}
\section{Convergence Guarantees of SVRPG}\label{sec:conv}
\vspace{-0.05in}
In this section, we state the convergence guarantees for SVRPG with REINFORCE or G(PO)MDP gradient estimator.
We mainly leverage on the recent analysis of non-concave SVRG~\cite{reddi2016stochastic,allen2016variance}.
Each of the three challenges presented at the beginning of Section~\ref{sec:alg} can potentially prevent convergence, so we need additional assumptions.
In Appendix~\ref{app:gauss} we show how Gaussian policies satisfy these assumptions.

\textit{1) Non-concavity.} A common assumption, in this case, is to assume the objective function to be $L$-smooth.
However, in RL we can consider the following assumption which is sufficient for the $L$-smoothness of the objective (see Lemma~\ref{lemma:lsmooth}).
	\begin{restatable}[On policy derivatives]{assumption}{boundedscore}\label{ass:bounded_score}
		For each state-action pair $(s,a)$, any value of $\vtheta$, and all parameter components $i,j$ there exist constants $0 \leq G,F<\infty$ such that:
\[
		\left|\nabla_{\theta_i}\log\pi_{\vtheta}(a\vert s)\right| \leq \GRADLOG, \qquad
        \left|\frac{\partial^2}{\partial\theta_i\partial\theta_j}\log\pi_{\vtheta}(a \vert s)\right| \leq \HESSLOG.
\]
	\end{restatable}

\textit{2) FG Approximation.}
Since we cannot compute an exact full gradient, we require the variance of the estimator to be bounded.
This assumption is similar in spirit to the one in~\citep{harikandeh2015stopwasting}.
	\begin{restatable}[On the variance of the gradient estimator]{assumption}{varreinforce}\label{ass:REINFORCE}
		There is a constant $V<\infty$ such that, for any policy $\pol$:
		\[
			\Var\left[g(\cdot\vert\vtheta)\right] \leq \VARRF.
		\]
	\end{restatable}

\textit{3) Non-stationarity.} 
Similarly to what is done in SL~\citep{cortes2010learning}, we require the variance of the importance weight to be bounded.
	\begin{restatable}[On the variance of importance weights]{assumption}{varweights}\label{ass:M2}
		There is a constant $W<\infty$ such that, for each pair of policies encountered in Algorithm~\ref{alg:svrpg} and for each trajectory,
		\[
                \mathbb{V}ar\left[\omega(\tau| \vtheta_1, \vtheta_2)\right] \leq \VARIS, \quad \forall \vtheta_1,\vtheta_2 \in \realspace^d , \tau \sim p(\cdot|\vtheta_1).
		\]
	\end{restatable}
Differently from Assumptions~\ref{ass:bounded_score} and ~\ref{ass:REINFORCE}, Assumption~\ref{ass:M2} must be enforced by a proper handling of the epoch size $m$.

We can now state the convergence guarantees for SVRPG.
\begin{restatable}[Convergence of the SVRPG algorithm]{theorem}{convergence}\label{theo:convergence}
Assume the REINFORCE or the G(PO)MDP gradient estimator is used in SVRPG (see Equation~\eqref{E:svrpg.estimate.batch}).
Under Assumptions~\ref{ass:bounded_score}, \ref{ass:REINFORCE} and \ref{ass:M2}, the parameter vector $\vtheta_A$ returned by Algorithm~\ref{alg:svrpg} after $T=m\times S$ iterations has, for some positive constants $\psi,\zeta, \xi$ and for proper choice of the step size $\alpha$ and the epoch size $m$, the following property:
\begin{align*}
	&\EVV[]
	{\norm[2]{\nabla J(\vtheta_A)}^2} 
		\leq
		\frac{J(\vtheta^*)-J(\vtheta_0)}{\psi T} +
		\frac{\zeta}{N}
		+\frac{\xi}{B},
\end{align*}
where $\vtheta^*$ is a global optimum and $\psi,\zeta,\xi$ depend only on $\GRADLOG,\HESSLOG,\VARRF,\VARIS,\alpha$ and $m$.
\end{restatable}
Refer to Appendix~\ref{app:proofs} for a detailed proof involving the definition of the constants and the meta-parameter constraints.
By analysing the upper-bound in Theorem~\ref{theo:convergence} we observe that: I) the $O(\nicefrac{1}{T})$ term is coherent with results on non-concave SVRG~\citep[\eg][]{reddi2016stochastic}; II) the $O(\nicefrac{1}{N})$ term is due to the FG approximation and is analogous to the one in~\citep{harikandeh2015stopwasting}; III) the $O(\nicefrac{1}{B})$ term is due to importance weighting.
To achieve asymptotic convergence, the batch size $N$ and the mini-batch size $B$ should increase over time. 
In practice, it is enough to choose $N$ and $B$ large enough to make the second and the third term negligible, \ie to mitigate the variance introduced by FG approximation and importance sampling, respectively.
Once the last two terms can be neglected, the number of trajectories needed to achieve $\norm[2]{\gradJ{\vtheta}}^2\leq\epsilon$ is $O(\frac{B+\nicefrac{N}{m}}{\epsilon})$. In this sense, an advantage over batch gradient ascent can be achieved with properly selected meta-parameters. In Section~\ref{sec:stopping} we propose a joint selection of step size $\alpha$ and epoch size $m$.
Finally, from the return statement of Algorithm \ref{alg:svrpg}, it is worth noting that $J(\vtheta_A)$ can be seen as the average performance of all the policies tried by the algorithm. This is particularly meaningful in the context of online learning that we are considering in this paper.

\vspace{-0.05in}
\section{Remarks on SVRPG}
\vspace{-0.05in}
The convergence guarantees presented in the previous section come with requirements on the meta-parameters (\ie $\alpha$ and $m$) that may be too conservative for practical applications.
Here we provide a practical and automatic way to choose the step size $\alpha$ and the number of sub-iterations $m$ performed between snapshots.
Additionally, we provide a variant of SVRPG exploiting a variance-reduction technique for importance weights.
Despite lacking theoretical guarantees, we will show in Section~\ref{sec:exp} that this method can outperform the baseline SVRPG (Algorithm~\ref{alg:svrpg}).

\vspace{-0.05in}
\subsection{Full Gradient Update}
\vspace{-0.05in}
As noted in Section \ref{sec:alg}, the update performed at the beginning of each epoch is equivalent to a full-gradient update. In our setting, where collecting samples is particularly expensive, the $B$ trajectories collected using the snapshot trajectory $\pi_{\wt{\vtheta}^s}$ feels like a waste of data (the term $\sum_i g(\tau_i) - \omega(\tau_i) g(\tau_i) =0$ since $\vtheta_0 = \wt{\vtheta}$).
In practice, we just perform an approximate full gradient update using the $N$ trajectories sampled to compute $\gradApp{\wt{\vtheta}^s}{N}$, \ie
\begin{align*}
	&\vtheta_{1}^{s+1} = \wt{\vtheta}^s + \alpha\gradApp{\wt{\vtheta}^s}{N} \\
	&\vtheta_{t+1}^{s+1} = \vtheta_t^{s+1} + \alpha 
        \blacktriangledown J(\vtheta^{s+1}_t)
        \text{ for $t=1,\dots,m-1$}.
\end{align*}
In the following, we will always use this practical variant.

\vspace{-0.05in}
\subsection{Meta-Parameter Selection}\label{sec:stopping}
\vspace{-0.05in}
The step size $\alpha$ is crucial to balance variance reduction and efficiency, while
the epoch length $m$ influences the variance introduced by the importance weights. Low values of $m$ are associated with small variance but increase the frequency of snapshot points (which means many FG computations). High values of $m$ may move policy $\pi_{\vtheta_t}$ far away from the snapshot policy $\pi_{\wt{\vtheta}}$, causing large variance in the importance weights. We will jointly set the two meta-parameters.

\textbf{Adaptive step size.}
A standard way to deal with noisy gradients is to use adaptive strategies to compute the step size.
ADAptive Moment estimation (ADAM)~\citep{kingma2014adam} stabilizes the parameter update by computing learning rates for each parameter based on an incremental estimate of the gradient variance.
Due to this feature, we would like to incorporate ADAM in the structure of the SVRPG update.
Recall that SVRPG performs two different updates of the parameters $\vtheta$: I) FG update in the snapshot; II) corrected gradient update in the sub-iterations.
Given this structure, we suggest using two separate ADAM estimators:
\begin{align*}
        \vtheta^{s+1}_1 &= \wt{\vtheta}^s + \alpha^{\textsc{FG}}_s\left(\wh{\nabla}_N J(\wt{\vtheta}^s) \right)\\
        \vtheta^{s+1}_{t+1} &= \vtheta^{s+1}_t + \alpha^{\textsc{SI}}_{s+1,t}\left( 
        \blacktriangledown J(\vtheta^{s+1}_t)\right)
        \text{ for $t=1,\dots,m-1$},
\end{align*}
where $\alpha^{\textsc{FG}}_{s}$ is associated with the snapshot and $\alpha^{\textsc{SI}}_{s+1,t}$ with the sub-iterations (see Appendix~\ref{app:practicalsvrpg} for details).
By doing so, we decouple the contribution of the variance due to the approximate FG from the one introduced by the sub-iterations.
Note that these two terms have different orders of magnitude since are estimated with a different number of trajectories ($B \ll N$) and the estimator in the snapshot does not require importance weights.
The use of two ADAM estimators allows to capture and exploit this property.

\textbf{Adaptive epoch length.}
It is easy to imagine that a predefined schedule (\eg $m$ fixed in advance or changed with a policy-independent process) may poorly perform due to the high variability of the updates.
In particular, given a fixed number of sub-iterations $m$, the variance of the updates in the sub-iterations depends on the snapshot policy and the sampled trajectories.
Since the ADAM estimate partly captures such variability,  we propose to take a new snapshot (\ie interrupt the sub-iterations) whenever the step size $\alpha^{\textsc{SI}}$ proposed by ADAM for the sub-iterations is smaller than the one for the FG (\ie $\alpha^{\textsc{FG}}$).
If the latter condition is verified, it amounts to say that the noise in the corrected gradient has overcome the information of the FG.
Formally, the stopping condition is as follows
\[
        \textbf{If }        \frac{\alpha^{\textsc{FG}}}{N} > \frac{\alpha^{\textsc{SI}}}{B} \textbf{ then } \text{take snapshot,}
\]
where we have introduced $N$ and $B$ to take into account the trajectory efficiency (\ie weighted advantage).
The less the number of trajectories used to update the policy, the better.
Including the batch sizes in the stopping condition allows us to optimize the trade-off between the quality of the updates and the cost of performing them.

\vspace{-0.05in}
\subsection{Normalized Importance Sampling}\label{sec:prac}
\vspace{-0.05in}
As mentioned in Section~\ref{sec:stopping}, importance weights are an additional source of variance. A standard way to cope with this issue is self-normalization~\citep[\eg][]{precup2000eligibility,owenmcbook}.
This technique can reduce the variance of the importance weights at the cost of introducing some bias~\citep[][Chapter 9]{owenmcbook}.
Whether the trade-off is advantageous depends on the specific task.  
Introducing self-normalization in the context of our algorithm, we switch from Eq.~\eqref{E:svrpg.estimate.batch} to:
\begin{align*}
        \blacktriangledown J(\vtheta_{t}) &= \wh{\nabla}_N J(\wt{\vtheta}) + \frac{1}{B} \sum_{i=0}^{B-1} \left[g(\tau_i|\vtheta_t)\right]\\ 
                                          &\qquad{} - \frac{1}{\Omega} \sum_{i=0}^{B-1} \left[ \omega(\tau_i|\vtheta_t, \wt{\vtheta}) g(\tau_i|\wt{\vtheta})
\right].
\end{align*}
where $\Omega = \sum_{i=0}^{B-1}\omega(\tau_i|\vtheta_t, \wt{\vtheta})$.
In Section \ref{sec:exp} we show that self-normalization can provide a performance improvement.

\vspace{-0.05in}
\section{Related Work}
\vspace{-0.05in}
Despite the considerable interest received in SL, variance-reduced gradient approaches have not attracted the RL community.
As far as we know, there are just two applications of SVRG in RL.
The first approach~\citep{du2017svrgpe} aims to exploit SVRG for policy evaluation.
The policy evaluation problem is more straightforward than the one faced in this paper (control problem).
In particular, since the goal is to evaluate just the performance of a predefined policy, the optimization problem is stationary.
The setting considered in the paper is the one of policy evaluation by minimizing the empirical mean squared projected Bellman error (MSPBE) with a linear approximation of the value function. \citet{du2017svrgpe} shown that this problem can be equivalently reformulated as a convex-concave saddle-point problem that is characterized by a finite-sum structure.
This problem can be solved using a variant of SVRG~\citep{Palaniappan2016svrgsaddle} for which convergence guarantees have been provided.
The second approach~\citep{xu2017svrgtrpo} uses SVRG as a practical method to solve the optimization problem faced by Trust Region Policy Optimization (TRPO) at each iteration. This is just a direct application of SVRG to a problem having finite-sum structure since no specific structure of the RL problem is exploited.
It is worth to mention that, for practical reasons, the authors proposed to use a Newton conjugate gradient method with SVRG.

In the recent past, there has been a surge of studies investigating variance reduction techniques for policy gradient methods.
The specific structure of the policy gradient allows incorporating a baseline (\ie a function $b :\mathcal{S}\times\mathcal{A} \to \realspace$) without affecting the unbiasedness of the gradient~\citep[\eg][]{williams1992simple,weaver2001optimal,peters2008reinforcement,Thomas2017actionbaseline,wu2018variance}.
Although the baseline can be arbitrarily selected, literature often refers to the optimal baseline as the one minimizing the variance of the estimate.
Nevertheless, even the baseline needs to be estimated from data. This fact may partially reduce its effectiveness by introducing variance.
Even if these approaches share the same goal as SVRG, they are substantially different.
In particular, the proposed SVRPG does not make explicit use of the structure of the policy gradient framework, and it is independent of the underlying gradient estimate (\ie with or without baseline).
This suggests that would be possible to integrate an ad-hoc SVRPG baseline to further reduce the variance of the estimate.
Since this paper is about the applicability of SVRG technique to RL, we consider this topic as future work.
Additionally, the experiments show that SVRPG has an advantage over G(PO)MPD even when the baseline is used (see the half-cheetah domain in Section~\ref{sec:exp}).

Concerning importance weighting techniques, RL has made extensive use of them for off-policy problems~\citep[\eg][]{precup2000eligibility,thomas2015high}. However, as mentioned before, SVRPG cannot be compared to such methods since it is in all respects an on-policy algorithm. Here, importance weighting is just a statistical tool used to preserve the unbiasedness of the corrected gradient.

\vspace{-0.05in}
\section{Experiments}\label{sec:exp}
\vspace{-0.05in}
In this section, we evaluate the performance of SVRPG and compare it with policy gradient (PG) on well known continuous RL tasks: Cart-pole balancing and Swimmer~\citep[\eg][]{duan2016benchmarking}.
We consider G(PO)MDP since it has a smaller variance than REINFORCE.
For our algorithm, we use a batch size $N=100$, a mini-batch size $B=10$, and the jointly adaptive step size $\alpha$ and epoch length $m$ proposed in Section \ref{sec:stopping}. Since the aim of this comparison is to show the improvement that SVRG-flavored variance reduction brings to SG in the policy gradient framework, we set the batch size of the baseline policy gradient algorithm to $B$. In this sense, we measure the improvement yielded by computing snapshot gradients and using them to adjust parameter updates. Since we evaluate on-line performance over the number of sampled trajectories, the cost of computing such snapshot gradients is automatically taken into consideration. To make the comparison fair, we also use Adam in the baseline PG algorithm, which we will denote simply as G(PO)MDP in the following.
In all the experiments, we use deep Gaussian policies with adaptive standard deviation (details on network architecture in Appendix \ref{app:exp}).
Each experiment is run $10$ times with a random policy initialization and seed, but this initialization is shared among the algorithms under comparison.
The length of the experiment, \ie the total number of trajectories, is fixed for each task. Performance is evaluated by using test-trajectories on a subset of the policies considered during the learning process. We provide average performance with 90\% bootstrap confidence intervals.  
Task implementations are from the \textit{rllab} library \cite{duan2016benchmarking}, on which our agents are also based.\footnote{Code available at \url{github.com/Dam930/rllab}.}
More details on meta-parameters and exhaustive task descriptions are provided in Appendix \ref{app:exp}.

\begin{figure*}[t]
    \begin{subfigure}[b]{0.49\textwidth}
	\includegraphics[width=.98\textwidth]{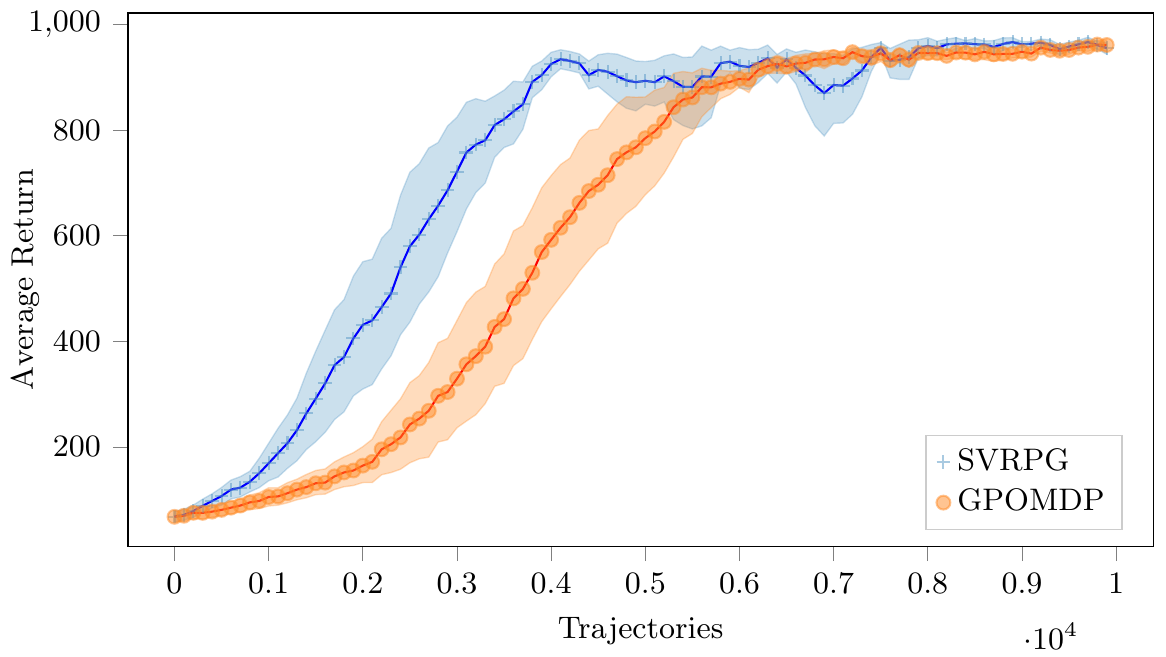}
	\vspace{-0.1in}
	\caption{SVRPG vs G(PO)MDP on Cart-pole.}
	\label{fig:cartpole}
    \end{subfigure}\hfill
    \begin{subfigure}[b]{0.49\textwidth}
	\includegraphics[width=0.98\textwidth]{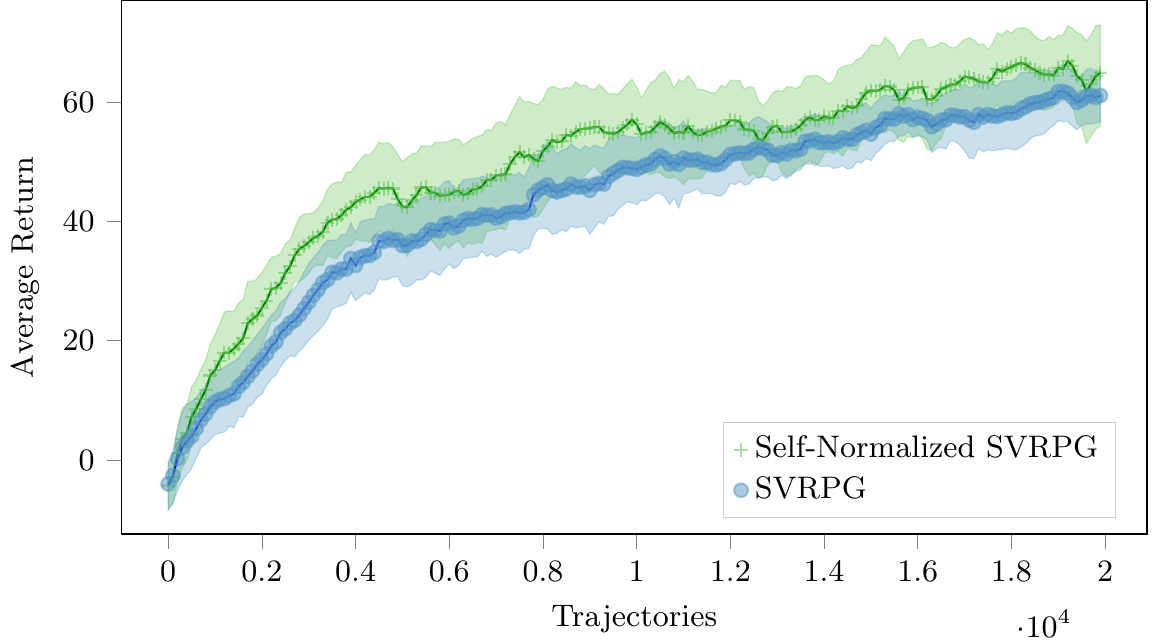}
	\vspace{-0.1in}
	\caption{Self-Normalized SVRPG vs SVRPG on Swimmer.}
	\label{fig:swimmertwo}
    \end{subfigure}
    
    \begin{subfigure}[b]{.49\textwidth}
	\includegraphics[width=0.98\textwidth]{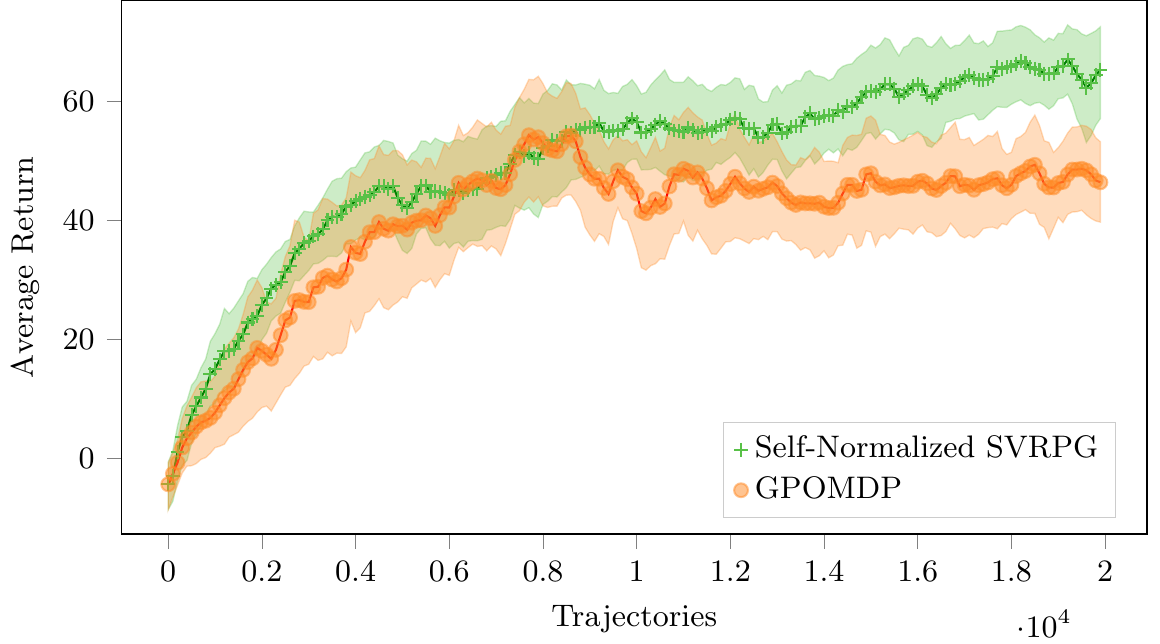}
	\vspace{-0.1in}
	\caption{Self-Normalized SVRPG vs G(PO)MDP on Swimmer.}
	\label{fig:swimmerone}
    \end{subfigure}\hfill
    \begin{subfigure}[b]{.49\textwidth}
    	\includegraphics[width=0.98\textwidth]{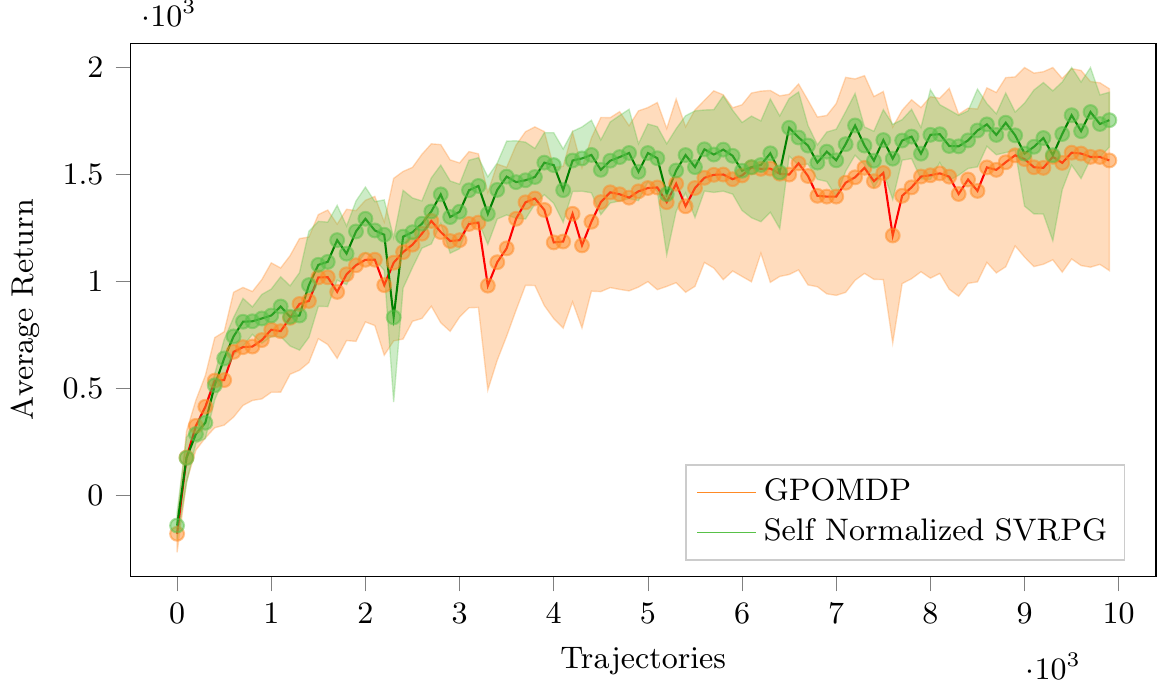}
    	\vspace{-0.1in}
    	\caption{Self-Normalized SVRPG vs G(PO)MDP on Half-Cheetah.}
    	\label{fig:cheetah}
    \end{subfigure}
    \caption{Comparison of on-line performance over sampled trajectories, with 90\% confidence intervals.}
\end{figure*}

%

Figure \ref{fig:cartpole} compares SVRPG with G(PO)MDP on a continuous variant of the classical Cart-pole task, which is a 2D balancing task. Despite using more trajectories on average for each parameter update, our algorithm shows faster convergence, which can be ascribed to the better quality of updates due to variance reduction.

The Swimmer task is a 3D continuous-control locomotion task. This task is more difficult than cart-pole. In particular, the longer horizon and the more complex dynamics can have a dangerous impact on the variance of importance weights. In this case, the self-normalization technique proposed in Section \ref{sec:prac} brings an improvement (even if not statistically significant), as shown in Figure \ref{fig:swimmertwo}. 
Figure \ref{fig:swimmerone} shows self-normalized SVRPG against G(PO)MDP. Our algorithm outperforms G(PO)MDP for almost the entire learning process. Also here, we note an increase of speed in early iterations, and, toward the end of the learning process, the improvement becomes statistically significant.


\textbf{Preliminary results on actor-critic.} 
Another variance-reduction technique in policy gradient consists of using baselines or \textit{critics}. This tool is orthogonal to the methods described in this paper, and the theoretical results of Section \ref{sec:conv} are general in this sense. In the experiments described so far, we compared against the so-called \textit{actor-only} G(PO)MDP, \ie without the baseline. To move towards a more general understanding of the variance issue in policy gradient, we also test SVRPG in an \textit{actor-critic} scenario. To do so, we consider the more challenging MuJoCo~\citep{todorov2012mujoco} Half-cheetah task, a 3D locomotion task that has a larger state-action space than Swimmer. Figure \ref{fig:cheetah} compares self-normalized SVRPG and G(PO)MDP on Half-cheetah, using the critic suggested in \cite{duan2016benchmarking} for both algorithms. Results are promising, showing that a combination of the baseline usage and SVRG-like variance reduction can yield an improvement that the two techniques alone are not able to achieve. Moreover, SVRPG presents a noticeably lower variance. The performance of actor-critic G(PO)MDP\footnote{\citet{duan2016benchmarking} report results on REINFORCE. However, inspection on \textit{rllab} code and documentation reveals that it is actually PGT \cite{sutton2000policy}, which is equivalent to G(PO)MDP \citep[shown by][]{peters2008reinforcement}. Using the name REINFORCE in a general way is inaccurate, but widespread.} on Half-Cheetah is coherent with the one reported in \cite{duan2016benchmarking}. Other results are not comparable since we did not use the critic.

\vspace{-0.05in}
\section{Conclusion}
\vspace{-0.05in}
In this paper, we introduced SVRPG, a variant of SVRG designed explicitly for RL problems.
The control problem considered in the paper has a series of difficulties that are not common in SL.
Among them, non-concavity and approximate estimates of the FG have been analysed independently in SL~\citep[\eg][]{allen2016variance,reddi2016stochastic,harikandeh2015stopwasting} but never combined.
Nevertheless, the main issue in RL is the non-stationarity of the sampling process since the distribution underlying the objective function is policy-dependent.
We have shown that by exploiting importance weighting techniques, it is possible to overcome this issue and preserve the unbiasedness of the corrected gradient.
We have additionally shown that, under mild assumptions that are often verified in RL applications, it is possible to derive convergence guarantees for SVRPG.
Finally, we have empirically shown that practical variants of the theoretical SVRPG version can outperform classical actor-only approaches on benchmark tasks.
Preliminary results support the effectiveness of SVRPG also with a commonly used baseline for the policy gradient.
Despite that, we believe that it will be possible to derive a baseline designed explicitly for SVRPG to exploit the RL structure and the SVRG idea jointly.
Another possible improvement would be to employ the natural gradient~\cite{kakade2002natural} to better control the effects of parameter updates on the variance of importance weights. Future work should also focus on making batch sizes $N$ and $B$ adaptive, as suggested in~\cite{papini2017adaptive}.

\FloatBarrier
\section*{Acknowledgments}
This research was supported in part by French Ministry of Higher Education and Research, Nord-Pas-de-Calais Regional Council and French National Research Agency (ANR) under project ExTra-Learn (n.ANR-14-CE24-0010-01).
\bibliography{svrpg}

\begin{thebibliography}{44}
\providecommand{\natexlab}[1]{#1}
\providecommand{\url}[1]{\texttt{#1}}
\expandafter\ifx\csname urlstyle\endcsname\relax
  \providecommand{\doi}[1]{doi: #1}\else
  \providecommand{\doi}{doi: \begingroup \urlstyle{rm}\Url}\fi

\bibitem[Allen-Zhu \& Hazan(2016)Allen-Zhu and Hazan]{allen2016variance}
Allen-Zhu, Zeyuan and Hazan, Elad.
\newblock Variance reduction for faster non-convex optimization.
\newblock In \emph{International Conference on Machine Learning}, pp.\
  699--707, 2016.

\bibitem[Baxter \& Bartlett(2001)Baxter and Bartlett]{baxter2001infinite}
Baxter, Jonathan and Bartlett, Peter~L.
\newblock Infinite-horizon policy-gradient estimation.
\newblock \emph{Journal of Artificial Intelligence Research}, 15:\penalty0
  319--350, 2001.

\bibitem[Bietti \& Mairal(2017)Bietti and Mairal]{bietti2017stochastic}
Bietti, Alberto and Mairal, Julien.
\newblock Stochastic optimization with variance reduction for infinite datasets
  with finite sum structure.
\newblock In \emph{Advances in Neural Information Processing Systems}, pp.\
  1622--1632, 2017.

\bibitem[Bottou \& LeCun(2004)Bottou and LeCun]{bottou2004large}
Bottou, L{\'e}on and LeCun, Yann.
\newblock Large scale online learning.
\newblock In \emph{Advances in neural information processing systems}, pp.\
  217--224, 2004.

\bibitem[Cauchy(1847)]{cauchy1847methode}
Cauchy, Augustin.
\newblock M{\'e}thode g{\'e}n{\'e}rale pour la r{\'e}solution des systemes
  d'{\'e}quations simultan{\'e}es.
\newblock \emph{Comp. Rend. Sci. Paris}, 25\penalty0 (1847):\penalty0 536--538,
  1847.

\bibitem[Cortes et~al.(2010)Cortes, Mansour, and Mohri]{cortes2010learning}
Cortes, Corinna, Mansour, Yishay, and Mohri, Mehryar.
\newblock Learning bounds for importance weighting.
\newblock In \emph{Advances in neural information processing systems}, pp.\
  442--450, 2010.

\bibitem[Defazio et~al.(2014{\natexlab{a}})Defazio, Bach, and
  Lacoste-Julien]{defazio2014saga}
Defazio, Aaron, Bach, Francis, and Lacoste-Julien, Simon.
\newblock Saga: A fast incremental gradient method with support for
  non-strongly convex composite objectives.
\newblock In \emph{Advances in Neural Information Processing Systems}, pp.\
  1646--1654, 2014{\natexlab{a}}.

\bibitem[Defazio et~al.(2014{\natexlab{b}})Defazio, Domke,
  et~al.]{defazio2014finito}
Defazio, Aaron, Domke, Justin, et~al.
\newblock Finito: A faster, permutable incremental gradient method for big data
  problems.
\newblock In \emph{International Conference on Machine Learning}, pp.\
  1125--1133, 2014{\natexlab{b}}.

\bibitem[Du et~al.(2017)Du, Chen, Li, Xiao, and Zhou]{du2017svrgpe}
Du, Simon~S., Chen, Jianshu, Li, Lihong, Xiao, Lin, and Zhou, Dengyong.
\newblock Stochastic variance reduction methods for policy evaluation.
\newblock In \emph{{ICML}}, volume~70 of \emph{Proceedings of Machine Learning
  Research}, pp.\  1049--1058. {PMLR}, 2017.

\bibitem[Duan et~al.(2016)Duan, Chen, Houthooft, Schulman, and
  Abbeel]{duan2016benchmarking}
Duan, Yan, Chen, Xi, Houthooft, Rein, Schulman, John, and Abbeel, Pieter.
\newblock Benchmarking deep reinforcement learning for continuous control.
\newblock In \emph{International Conference on Machine Learning}, pp.\
  1329--1338, 2016.

\bibitem[Furmston \& Barber(2012)Furmston and Barber]{Furmston2012unifying}
Furmston, Thomas and Barber, David.
\newblock A unifying perspective of parametric policy search methods for markov
  decision processes.
\newblock In \emph{{NIPS}}, pp.\  2726--2734, 2012.

\bibitem[Ghadimi \& Lan(2013)Ghadimi and Lan]{ghadimi2013stochastic}
Ghadimi, Saeed and Lan, Guanghui.
\newblock Stochastic first-and zeroth-order methods for nonconvex stochastic
  programming.
\newblock \emph{SIAM Journal on Optimization}, 23\penalty0 (4):\penalty0
  2341--2368, 2013.

\bibitem[Harikandeh et~al.(2015)Harikandeh, Ahmed, Virani, Schmidt,
  Kone{\v{c}}n{\`y}, and Sallinen]{harikandeh2015stopwasting}
Harikandeh, Reza, Ahmed, Mohamed~Osama, Virani, Alim, Schmidt, Mark,
  Kone{\v{c}}n{\`y}, Jakub, and Sallinen, Scott.
\newblock Stopwasting my gradients: Practical svrg.
\newblock In \emph{Advances in Neural Information Processing Systems}, pp.\
  2251--2259, 2015.

\bibitem[Johnson \& Zhang(2013)Johnson and Zhang]{johnson2013accelerating}
Johnson, Rie and Zhang, Tong.
\newblock Accelerating stochastic gradient descent using predictive variance
  reduction.
\newblock In \emph{Advances in neural information processing systems}, pp.\
  315--323, 2013.

\bibitem[Jur{\v{c}}{\'\i}{\v{c}}ek(2012)]{jurvcivcek2012reinforcement}
Jur{\v{c}}{\'\i}{\v{c}}ek, Filip.
\newblock Reinforcement learning for spoken dialogue systems using off-policy
  natural gradient method.
\newblock In \emph{Spoken Language Technology Workshop (SLT), 2012 IEEE}, pp.\
  7--12. IEEE, 2012.

\bibitem[Kakade(2002)]{kakade2002natural}
Kakade, Sham~M.
\newblock A natural policy gradient.
\newblock In \emph{Advances in neural information processing systems}, pp.\
  1531--1538, 2002.

\bibitem[Kingma \& Ba(2014)Kingma and Ba]{kingma2014adam}
Kingma, Diederik~P and Ba, Jimmy.
\newblock Adam: A method for stochastic optimization.
\newblock \emph{arXiv preprint arXiv:1412.6980}, 2014.

\bibitem[Kone{\v{c}}n{\`y} et~al.(2016)Kone{\v{c}}n{\`y}, Liu, Richt{\'a}rik,
  and Tak{\'a}{\v{c}}]{konevcny2016mini}
Kone{\v{c}}n{\`y}, Jakub, Liu, Jie, Richt{\'a}rik, Peter, and Tak{\'a}{\v{c}},
  Martin.
\newblock Mini-batch semi-stochastic gradient descent in the proximal setting.
\newblock \emph{IEEE Journal of Selected Topics in Signal Processing},
  10\penalty0 (2):\penalty0 242--255, 2016.

\bibitem[Mairal(2015)]{mairal2015incremental}
Mairal, Julien.
\newblock Incremental majorization-minimization optimization with application
  to large-scale machine learning.
\newblock \emph{SIAM Journal on Optimization}, 25\penalty0 (2):\penalty0
  829--855, 2015.

\bibitem[Nemirovskii et~al.(1983)Nemirovskii, Yudin, and
  Dawson]{nemirovskii1983problem}
Nemirovskii, Arkadii, Yudin, David~Borisovich, and Dawson, Edgar~Ronald.
\newblock Problem complexity and method efficiency in optimization.
\newblock 1983.

\bibitem[Nesterov(2013)]{nesterov2013introductory}
Nesterov, Yurii.
\newblock \emph{Introductory lectures on convex optimization: A basic course},
  volume~87.
\newblock Springer Science \& Business Media, 2013.

\bibitem[Owen(2013)]{owenmcbook}
Owen, Art~B.
\newblock \emph{Monte Carlo theory, methods and examples}.
\newblock 2013.

\bibitem[Palaniappan \& Bach(2016)Palaniappan and
  Bach]{Palaniappan2016svrgsaddle}
Palaniappan, Balamurugan and Bach, Francis.
\newblock Stochastic variance reduction methods for saddle-point problems.
\newblock In \emph{{NIPS}}, pp.\  1408--1416, 2016.

\bibitem[Papini et~al.(2017)Papini, Pirotta, and Restelli]{papini2017adaptive}
Papini, Matteo, Pirotta, Matteo, and Restelli, Marcello.
\newblock Adaptive batch size for safe policy gradients.
\newblock In \emph{Advances in Neural Information Processing Systems}, pp.\
  3594--3603, 2017.

\bibitem[Peters \& Schaal(2008{\natexlab{a}})Peters and
  Schaal]{Peters2008reinf}
Peters, Jan and Schaal, Stefan.
\newblock Reinforcement learning of motor skills with policy gradients.
\newblock \emph{Neural Networks}, 21\penalty0 (4):\penalty0 682--697,
  2008{\natexlab{a}}.

\bibitem[Peters \& Schaal(2008{\natexlab{b}})Peters and
  Schaal]{peters2008reinforcement}
Peters, Jan and Schaal, Stefan.
\newblock Reinforcement learning of motor skills with policy gradients.
\newblock \emph{Neural networks}, 21\penalty0 (4):\penalty0 682--697,
  2008{\natexlab{b}}.

\bibitem[Pirotta et~al.(2013)Pirotta, Restelli, and
  Bascetta]{pirotta2013adaptive}
Pirotta, Matteo, Restelli, Marcello, and Bascetta, Luca.
\newblock Adaptive step-size for policy gradient methods.
\newblock In \emph{Advances in Neural Information Processing Systems}, pp.\
  1394--1402, 2013.

\bibitem[Pirotta et~al.(2015)Pirotta, Restelli, and
  Bascetta]{pirotta2015lipschitz}
Pirotta, Matteo, Restelli, Marcello, and Bascetta, Luca.
\newblock Policy gradient in lipschitz markov decision processes.
\newblock \emph{Machine Learning}, 100\penalty0 (2):\penalty0 255--283, 2015.
\newblock ISSN 1573-0565.
\newblock \doi{10.1007/s10994-015-5484-1}.

\bibitem[Precup(2000)]{precup2000eligibility}
Precup, Doina.
\newblock Eligibility traces for off-policy policy evaluation.
\newblock \emph{Computer Science Department Faculty Publication Series}, pp.\
  ~80, 2000.

\bibitem[Reddi et~al.(2016{\natexlab{a}})Reddi, Hefny, Sra, Poczos, and
  Smola]{reddi2016stochastic}
Reddi, Sashank~J, Hefny, Ahmed, Sra, Suvrit, Poczos, Barnabas, and Smola, Alex.
\newblock Stochastic variance reduction for nonconvex optimization.
\newblock In \emph{International conference on machine learning}, pp.\
  314--323, 2016{\natexlab{a}}.

\bibitem[Reddi et~al.(2016{\natexlab{b}})Reddi, Sra, P{\'o}czos, and
  Smola]{reddi2016fast}
Reddi, Sashank~J, Sra, Suvrit, P{\'o}czos, Barnab{\'a}s, and Smola, Alex.
\newblock Fast incremental method for nonconvex optimization.
\newblock \emph{arXiv preprint arXiv:1603.06159}, 2016{\natexlab{b}}.

\bibitem[Robbins \& Monro(1951)Robbins and Monro]{robbins1951stochastic}
Robbins, Herbert and Monro, Sutton.
\newblock A stochastic approximation method.
\newblock \emph{The annals of mathematical statistics}, pp.\  400--407, 1951.

\bibitem[Roux et~al.(2012)Roux, Schmidt, and Bach]{roux2012stochastic}
Roux, Nicolas~L, Schmidt, Mark, and Bach, Francis~R.
\newblock A stochastic gradient method with an exponential convergence \_rate
  for finite training sets.
\newblock In \emph{Advances in Neural Information Processing Systems}, pp.\
  2663--2671, 2012.

\bibitem[Rubinstein(1981)]{rubinstein1981simulation}
Rubinstein, Reuven Y Reuven~Y.
\newblock Simulation and the monte carlo method.
\newblock Technical report, 1981.

\bibitem[Sutton \& Barto(1998)Sutton and Barto]{sutton1998reinforcement}
Sutton, Richard~S and Barto, Andrew~G.
\newblock \emph{Reinforcement learning: An introduction}.
\newblock MIT press Cambridge, 1998.

\bibitem[Sutton et~al.(2000)Sutton, McAllester, Singh, and
  Mansour]{sutton2000policy}
Sutton, Richard~S, McAllester, David~A, Singh, Satinder~P, and Mansour, Yishay.
\newblock Policy gradient methods for reinforcement learning with function
  approximation.
\newblock In \emph{Advances in neural information processing systems}, pp.\
  1057--1063, 2000.

\bibitem[Thomas \& Brunskill(2017)Thomas and
  Brunskill]{Thomas2017actionbaseline}
Thomas, Philip~S. and Brunskill, Emma.
\newblock Policy gradient methods for reinforcement learning with function
  approximation and action-dependent baselines.
\newblock \emph{CoRR}, abs/1706.06643, 2017.

\bibitem[Thomas et~al.(2015)Thomas, Theocharous, and
  Ghavamzadeh]{thomas2015high}
Thomas, Philip~S, Theocharous, Georgios, and Ghavamzadeh, Mohammad.
\newblock High-confidence off-policy evaluation.
\newblock In \emph{AAAI}, pp.\  3000--3006, 2015.

\bibitem[Todorov et~al.(2012)Todorov, Erez, and Tassa]{todorov2012mujoco}
Todorov, Emanuel, Erez, Tom, and Tassa, Yuval.
\newblock Mujoco: A physics engine for model-based control.
\newblock In \emph{Intelligent Robots and Systems (IROS), 2012 IEEE/RSJ
  International Conference on}, pp.\  5026--5033. IEEE, 2012.

\bibitem[Weaver \& Tao(2001)Weaver and Tao]{weaver2001optimal}
Weaver, Lex and Tao, Nigel.
\newblock The optimal reward baseline for gradient-based reinforcement
  learning.
\newblock In \emph{Proceedings of the Seventeenth conference on Uncertainty in
  artificial intelligence}, pp.\  538--545. Morgan Kaufmann Publishers Inc.,
  2001.

\bibitem[Williams(1992)]{williams1992simple}
Williams, Ronald~J.
\newblock Simple statistical gradient-following algorithms for connectionist
  reinforcement learning.
\newblock \emph{Machine learning}, 8\penalty0 (3-4):\penalty0 229--256, 1992.

\bibitem[Wu et~al.(2018)Wu, Rajeswaran, Duan, Kumar, Bayen, Kakade, Mordatch,
  and Abbeel]{wu2018variance}
Wu, Cathy, Rajeswaran, Aravind, Duan, Yan, Kumar, Vikash, Bayen, Alexandre~M,
  Kakade, Sham, Mordatch, Igor, and Abbeel, Pieter.
\newblock Variance reduction for policy gradient with action-dependent
  factorized baselines.
\newblock \emph{International Conference on Learning Representations}, 2018.
\newblock accepted as oral presentation.

\bibitem[Xu et~al.(2017)Xu, Liu, and Peng]{xu2017svrgtrpo}
Xu, Tianbing, Liu, Qiang, and Peng, Jian.
\newblock Stochastic variance reduction for policy gradient estimation.
\newblock \emph{CoRR}, abs/1710.06034, 2017.

\bibitem[Zhao et~al.(2011)Zhao, Hachiya, Niu, and Sugiyama]{zhao2011analysis}
Zhao, Tingting, Hachiya, Hirotaka, Niu, Gang, and Sugiyama, Masashi.
\newblock Analysis and improvement of policy gradient estimation.
\newblock In \emph{Advances in Neural Information Processing Systems}, pp.\
  262--270, 2011.

\end{thebibliography}
\bibliographystyle{icml2018}

\clearpage
\onecolumn
\appendix

\section{Policy Gradient Estimators} \label{A:gradient_estimators}
We present a brief overview of the two most widespread gradient estimators (REINFORCE~\citep{williams1992simple} and G(PO)MDP~\citep{baxter2001infinite}) both in on-policy and off-policy settings.
Let $\tau = \{\langle s_t,a_t \rangle\}_{t=0}^{H}= \{z_t\}_{t=0}^{H} = z_{0:H}$ is a $(H+1)$-steps trajectory.
Given that $\tau$ depends on the MDP $M=\{\mathcal{S},\mathcal{A}, \mathcal{P}, \Reward, \gamma, \rho\}$ and the actual policy $\pi$, the trajectory is
said drawn from density distribution $p(\tau|\pi,M)$ defined as:
\[
        p(\tau|\pi,M) = \rho(s_0) \pi(z_0) \prod_{k=1}^{H} \mathcal{P}(s_k|z_{k-1})\pi(z_k).
\]
We can now recall the definition of policy performance 
\[
        J(\pol) = \EVV[\tau \sim p(\cdot|\pol)]{\Reward(\tau)|M},
\]
where $\Reward(\tau) = \sum_{t=0}^{H}\gamma^t \Reward(z_t)$.
The policy gradient $\nabla J(\vtheta)$ is
\begin{equation}\label{E:onpolicygradient}
        \gradJ{\vtheta} = \EVV[\tau \sim p(\cdot|\pol)]{\nabla \log p(\tau|\vtheta)\Reward(\tau)} = \EVV[\tau \sim p(\cdot|\pol)]{\sum_{j=0}^{H} \gamma^j \mathcal{R}(z_j) \sum_{i=0}^{H} \nabla \log\pi(z_i)}.
\end{equation}

\textbf{On-policy setting.}
Consider a policy $\pol$ and let $\mathcal{D}_N = \{\tau_i\}_{i=1}^N$ be a dataset collected using policy $\pol$.
The REINFORCE gradient estimator~\citep{williams1992simple} provides a simple, unbiased way of estimating the gradient:
\begin{align*}
\gradApp{\vtheta}{N} = \frac{1}{N}\sum_{n=1}^{N}
\underbracket{
        \left(\sum_{h=0}^{H}\nabla \log \pol(z_h^n) \right)\left(\sum_{h=0}^{H}\gamma^h 
        r_h^n
- b(z_h^n)\right)
}_{g(\tau_n|\vtheta):=\nabla \log p(\tau_n|\vtheta)\Reward(\tau_n)}
,
\end{align*}
where subscripts denote the time step, superscripts denote the trajectory, $r_h^n$ is the reward actually collected at time $h$ from trajectory $\tau^n$ and $b : \mathcal{S}\times\mathcal{A} \to \realspace$~\citep[\eg][]{Thomas2017actionbaseline}.
The G(PO)MDP gradient estimator~\cite{baxter2001infinite} is a refinement of REINFORCE which is subject to less variance \cite{zhao2011analysis} while preserving the unbiasedness:
\begin{align*}
\gradApp{\vtheta}{N} = \frac{1}{N}\sum_{n=1}^{N}
\underbracket{
        \sum_{h=0}^{H}\left(\sum_{k=0}^{h} \nabla \log \pol (z_h^n) \right)\left(\gamma^h 
        r_h^n
        - b(z^n_h)\right)
}_{g(\tau_n|\vtheta)}.
\end{align*}
G(PO)MDP can be seen as a more efficient implementation of the REINFORCE algorithm. 
In fact, the latter does not perform an optimal credit assignment since it ignores that the reward at time $t$ does not depend on the action performed after time $t$.
G(PO)MDP overcomes this issue taking into account the causality of rewards in the REINFORCE definition of policy gradient.

\textbf{Off-policy setting.}
In off-policy setting two policies, called behavioural $\pi^B$ and target $\pi^T$, are involved.
The first is used to select actions for the interaction with the system, while the second is used to evaluate the agent performance and it is improved in each update.
Suppose now that we aim to estimate the performance of the target policy $\pi^T$ but we have samples collected using policy $\pi^B$.
We can use importance weight correction to correct the shift in the distribution and obtain an unbiased estimate of $J(\pi^T)$:
\[
        J(\pi^T) = \EVV[\tau \sim p(\cdot|\pi^T)]{\mathcal{R}(\tau)} = \EVV[\tau \sim p(\cdot|\pi^B)]{\omega(\tau) \mathcal{R}(\tau)}
\]
where $\omega(\tau|\pi^B,\pi^T) = \frac{p(\tau,\pi^T)}{p(\tau|\pi^B)} = \omega(z_{0:H}|\pi^B,\pi^T) = \prod_{w=0}^{H} \omega(z_w|\pi^B,\pi^T)$ and $\omega(z_w|\pi^B,\pi^T) = \frac{\pi^T(a_w|s_w)}{\pi^B(a_w|s_w)}$

The definition of the off-policy version of~\eqref{E:onpolicygradient} is~\citep[\eg][]{jurvcivcek2012reinforcement}
\begin{equation}\label{E:offpolicygradient}
        \nabla J(\pi^T) = \EVV[\tau \sim p(\cdot|\pi^B)]{\omega(\tau|\pi^B,\pi^T) \nabla \log p(\tau|\pi^T) \mathcal{R}(\tau)} = \EVV[\tau \sim p(\cdot|\pi^B)]{\omega(\tau|\pi^B,\pi^T) g(\tau|\pi^T)}. 
\end{equation}
For $\omega(\tau|\pi^B,\pi^T)$ being well defined the behavioural policy should have non-zero probability of selecting any action in every state \ie $\pi^B(a|s) > 0$ for any $(s,a)\in \mathcal{S} \times \mathcal{A}$.
Equation~\ref{E:offpolicygradient} is important for proving Theorem~\ref{theo:convergence} since it provides a common representation of REINFORCE and G(PO)MDP.

The off-policy version of REINFORCE is easily obtained by taking the empirical average of~\eqref{E:offpolicygradient}:
\[
        \nabla J(\pi^T) = \frac{1}{N} \sum_{n=1}^{N} \omega(\tau^n|\pi^B, \pi^T)
        \underbracket{
                \left(\sum_{h=0}^{H} \nabla \log \pi^{T}(z_h^n) \right)\left(\sum_{h=0}^{H}\gamma^h \Reward(z_h^n)\right)
        }_{g(\tau_n|\pi^T)}.
\]
The G(PO)MDP off-policy estimator is defined as follows
\[
        \nabla J(\pi^T) = \frac{1}{N} \sum_{n=1}^{N}
        \underbracket{
                \sum_{h=0}^H \left(\sum_{k=0}^h \nabla \log\pi^T(z_k^n)\right) \gamma^h \Reward(z_h^n) \omega(z_{0:h}|\pi^B,\pi^T)
        }_{\omega(\tau_n|\pi^B,\pi^T)g(\tau_n|\pi^T)}.
\]


\section{Proofs}\label{app:proofs}
In this section, we prove all the claims made in the paper, with the primary objective of proving Theorem \ref{theo:convergence}.
Our proof is adapted from the one of Theorem 2 from \cite{reddi2016stochastic} and has a very similar structure, but with all the additional challenges and assumptions described in Section \ref{sec:conv}.

Note that in the following we will make wide use of the following properties.
\begin{assumption}\label{asm:unbiasedness}
        We consider an estimate $\wh{\nabla}_N J(\vtheta)$ as in Eq.~\ref{E:policygradient.estimate} such that
\begin{enumerate}
        \item \textit{On-policy Unbiased Estimator.} 
                \[
                        \EVV[\tau_i \sim \pi(\cdot|\vtheta)]{\wh{\nabla}_N J(\vtheta)} = \EVV[\tau_i \sim \pi(\cdot|\vtheta)]{\frac{1}{N} \sum_{i=0}^N g(\tau_i|\vtheta)} = \nabla J(\vtheta)
                \]
        \item \textit{Off-policy Unbiased Estimator.}
                \[
        \EVV[\tau_i \sim \Delta(\cdot|\pi^B)]{\frac{1}{N} \sum_{i=0}^N \omega(\tau_i|\pi^B,\vtheta) g(\tau_i|\vtheta) } 
        = \nabla J(\vtheta) 
                \]
\end{enumerate}
\end{assumption}
Note that these assumptions are verified by REINFORCE and G(PO)MDP.

\subsection*{Definitions}
We give some additional definitions which will be useful in the proofs.

\begin{definition}
For a random variable X:
\[
	\Es{X} = \EVV[\tau_j\sim\pi(\cdot\vert\wt{\vtheta}^s)
		\forall j\in\mathbf{N}]{X\vert\wt{\vtheta}^s},
\]
where $\wt{\vtheta}^s$ is defined in Algorithm \ref{alg:svrpg} and $\mathbf{N} = [0,\dots,N)$.
\end{definition}

We introduce the notation $\tau_{i,h}$ to denote the $h$-th trajectory collected using policy $\vtheta^{s+1}_i$ where $s$ will be clear from the context.

\begin{definition}
For a random variable $X$:
\begin{align*}
        \mathbb{E}_{t\vert s}\left[X\right] &\coloneqq 
		\mathop{\mathbb{E}}_{\substack{\tau_j\sim\pi(\cdot\vert\wt{\vtheta}^s)\forall j \in {\it {\bf N}} \\ \tau_{i,h}\sim\pi(\cdot\vert\vtheta^{s+1}_i) \forall h \in {\it {\bf B}}, \text{ for $i=0,\dots,t$}}}{\left[X \vert \wt{\vtheta^s}\right]} \\
	&\coloneqq \EVV[\tau_j\sim\pi(\cdot\vert\wt{\vtheta}^s)\forall j \in {\it {\bf N}}]{
			\EVV[\tau_{0,h}\sim\pi(\cdot\vert\vtheta_0^{s+1}) \forall h \in {\it {\bf B}}]
				{\dots
					\EVV[\tau_{t,h}\sim\pi(\cdot\vert\vtheta_t^{s+1})\forall h \in {\it {\bf B}}]
						{X\vert\vtheta_t^{s+1}}
				 \dots
			\vert\vtheta_0^{s+1}}
		\vert\wt{\vtheta}^s}\\
        &= \EVV[t-1|s]{\EVV[\tau_{t,h}\sim \pi(\cdot|\vtheta^{s+1}_t)]{X|\vtheta^{s+1}_t}}
\end{align*}
where the sequence $\wt{\vtheta}^s,\vtheta_0^{s+1},\dots,\vtheta_t^{s+1}$ is defined in Algorithm \ref{alg:svrpg}, ${\it {\bf N}} = [0,\dots,N)$, and ${\it {\bf B}} = [0,\dots,B)$. To avoid inconsistencies, we also define $\Ets[(-1)]{X} \coloneqq \Es{X}$.

\end{definition}
Intuitively, the $\Ets{\cdot}$ operator computes the expected value with respect to the sampling of trajectories from the snapshot $\wt{\vtheta}^s$ up to the $t$-th iteration included. Note that the order in which expected values are taken is important since each $\vtheta_{t}^{s+1}$ is function of previously sampled trajectories and is used to sample new ones.

\begin{definition}\label{def:var}
For random vectors X, Y:
\begin{align*}
	\Covs{X}{Y} &\coloneqq \Tr\left(\Ets{(X-\Ets{X})(Y-\Ets{Y})^T}\right), \\
	\Vars{X} &\coloneqq \Covs{X}{Y},
\end{align*}
where $\Tr(\cdot)$ denotes the trace of a matrix. From the linearity of expected value we have the following:
\begin{equation}\label{neweq:1}
\Vars{X} = \Es{\norm[]{X-\Es{X}}^2}
\end{equation}
$\Covts[t]{X}{Y}$ and $\Varts[t]{X}$ are defined in the same way from $\Ets[t]{X}$.
\end{definition}

\begin{definition}
The full gradient estimation error is:
\[
	e_s \coloneqq \gradApp{\wt{\vtheta}^s}{N} - \gradJ{\wt{\vtheta}^s} 
\]
\end{definition}

\begin{definition}\label{def:ideal}
The ideal SVRPG gradient estimate is:
\begin{align*}
	\gradIdeal{\vtheta_t^{s+1}} &\coloneqq 
	\gradJ{\wt{\vtheta}^s}
    + g(\tau_i|\vtheta^{s+1}_t)
    - \omega(\tau_i|\vtheta^{s+1}_t, \wt{\vtheta}^s) g(\tau_i|\wt{\vtheta}^s)
    \\
	&= \gradBlack{\vtheta_t^{s+1}} - \gradApp{\wt{\vtheta}^s}{N} + \gradJ{\wt{\vtheta}^s} \\
	&= \gradBlack{\vtheta_t^{s+1}} - e_s
\end{align*}
\end{definition}

\subsection*{Basic Lemmas}
We prove two basic properties of the SVRPG update.

\svrpgprop*
\begin{proof}
\begin{align*}
        \EVV[]{\gradBlack{\vtheta}} &= \EVV[]{\gradApp{\wt{\vtheta}}{N}}  + \EVV[]{\gradApp{\vtheta}{B}} - \EVV[]{\frac{1}{B}\sum_{i=0}^{B-1}\omega(\tau_i|\vtheta, \wt{\vtheta}) g(\tau_i|\wt{\vtheta})} \\
&= \gradJ{\wt{\vtheta}} + \gradJ{\vtheta} - \gradJ{\wt{\vtheta}} = \gradJ{\vtheta}.
\end{align*}
Note that the importance weight is necessary to guarantee unbiasedness, since the $\tau_i$ are sampled from $\pi_{\vtheta}$.
As $\vtheta\to\vtheta^*$, also $\wt{\vtheta}\to\vtheta^*$. Hence, by continuity of $J(\vtheta)$:
	\begin{align*}
\Var\left[\gradBlack{\vtheta}\right] &\to \Var\left[\gradApp{\vtheta^*}{N}\right] + \frac{1}{B}\Var\left[g(\tau|\vtheta^*) - \cancel{\omega(\tau|\vtheta^*,\vtheta^*)}g(\tau|\vtheta^*)\right] \\
&= \Var\left[\gradApp{\vtheta^*}{N}\right].
\end{align*}
Note that it is important that the trajectories used in the second and the third term are the same for the variance to vanish.
\end{proof}

\subsection*{Ancillary Lemmas}
Before addressing the main convergence theorem, we prove some useful lemmas.

\begin{restatable}[]{lemma}{L-smoothness}\label{lemma:lsmooth}
	Under Assumption \ref{ass:bounded_score}, $J(\vtheta)$ is L-smooth for some positive Lipschitz constant $L_J$.
\end{restatable}
\begin{proof}
By definition of $J(\vtheta)$:
\begin{align}
\Dij{J(\vtheta)}{\theta} 
&= \int_{\Tspace}\Dij{}{\theta}p(\tau|\vtheta)\Reward(\tau)\de \tau
\nonumber\\ 
&= \int_{\Tspace}p(\tau|\vtheta)\score{\vtheta}{\tau}\score{\vtheta}{\tau}^T\Reward(\tau)\de \tau + \int_{\Tspace}\pol(\tau)\Dij{}{\theta}\log p(\tau|\vtheta)\Reward(\tau)\de \tau \nonumber\\
&\leq \sup_{\tau \in \mathcal{T}} \left\{\left|\Reward(\tau)\right|\right\} \left(H^2\GRADLOG^2+H\HESSLOG\right) \label{eq:0}\\
&= \frac{1-\gamma^H}{1-\gamma}RH\left(H\GRADLOG^2+\HESSLOG\right),\nonumber
\end{align}
where \ref{eq:0} is from Assumption \ref{ass:bounded_score}.
Since the Hessian is bounded, $J(\vtheta)$ is Lipschitz-smooth.
\end{proof}

\begin{restatable}[]{lemma}{l-smooth-g}\label{lemma:gsmooth}
Under Assumption \ref{ass:bounded_score}, whether we use the REINFORCE or the G(PO)MDP gradient estimator, $g(\tau\vert\vtheta)$ is Lipschitz continuous with Lipschitz constant $L_g$, \ie for any trajectory $\tau\in\Tspace$:
\[
	\norm[2]{g(\tau\vert\vtheta)-g(\tau\vert\vtheta')}^2 \leq L_g\norm[2]{\vtheta-\vtheta'}^2.
\]
\end{restatable}
\begin{proof}
        For both REINFORCE and G(PO)MDP, $g(\tau\vert\vtheta)$ is a linear combination of terms of the kind $\nabla \log \pi_{\vtheta}(a_t\vert s_t)\gamma^t r_t$ \cite{peters2008reinforcement}. These terms have bounded gradient from the second inequality of Assumption \ref{ass:bounded_score} and the fact that $|r_t|\leq R$. If a baseline is used in REINFORCE or G(PO)MDP, we only need the additional assumption that said baseline is bounded.
        Bounded gradient implies Lipschitz continuity. Finally, the linear combination of Lipschitz continuous functions is Lipschitz continuous.
\end{proof}

\begin{restatable}[]{lemma}{bounded-g}\label{lemma:gbound}
Under Assumption \ref{ass:bounded_score}, whether we use the REINFORCE or the G(PO)MDP gradient estimator, for every $\tau\in\Tspace$ and $\vtheta\in\Theta$, there is a positive constant $\Gamma<\infty$ such that:
\[
	\norm[2]{g(\tau\vert\vtheta)}^2 \leq \Gamma.
\]
\end{restatable}
\begin{proof}
For REINFORCE we have, from Assumption \ref{ass:bounded_score}:
\begin{align*}
	\norm[2]{g(\tau\vert\vtheta)}^2 &=
	\norm[2]{\score{\vtheta}{\tau}R(\tau)}^2 \\
	&=\norm[2]{\left(\sum_{t=0}^{H-1}\score{\vtheta}{a_t\vert s_t}\right)\left(\sum_{t=0}^{H-1}\gamma^t r_t\right)}^2\leq H^2G^2\frac{(1-\gamma^H)^2}{(1-\gamma)^2}R^2\dim(\vTheta) \coloneqq \Gamma
\end{align*}
For G(PO)MDP, we do not have a compact expression for $g$, but since it is derived from REINFORCE by neglecting some terms of the kind $\nabla \log \pi_{\vtheta}(a_t | s_t)\gamma^t r_t$ \cite{baxter2001infinite,peters2008reinforcement}, the above bound still holds.
If a baseline is used in REINFORCE or G(PO)MDP, we only need the additional assumption that said baseline is bounded.
\end{proof}

\begin{restatable}[]{lemma}{varineq}\label{lemma:varineq}
For any random vector X, the variance (as defined in Definition \ref{def:var}), can be bounded as follows:
\[
	\Vars{X} \leq \Es{\norm[]{X}^2}.
\]
\end{restatable}
\begin{proof}
By using basic properties of expected value and scalar variance:
\begin{align*}
	\Vars{X} = \Es{\norm[]{X-\Es{X}}^2} &= \Es{\sum_{i=1}^{\dim(X)}\left(X_i-\Es{X_i}\right)^2} = \sum_{i=1}^{\dim(X)}\Es{\left(X_i-\Es{X_i}\right)^2} \\
	&\leq \sum_{i=1}^{\dim(X)}\Es{X_i^2} = \Es{\sum_{i=1}^{\dim(X)}X_i^2} = \Es{\norm[]{X}^2}.
\end{align*}
\end{proof}

\begin{restatable}[]{lemma}{auxtwo}\label{lemma:aux2}
Under Assumption \ref{ass:bounded_score}
, the expected squared norm of the SVRPG gradient can be bounded as follows:
\[
        \Ets{\norm[2]{\gradBlack{\vtheta_t^{s+1}}}^2} \leq
\Ets[t-1]{\norm[2]{\gradJ{\vtheta_t^{s+1}}}^2} 
+\frac{L_g^2}{B}\Ets[t-1]{\norm[2]{\vtheta_t^{s+1}-\wt{\vtheta}^s}^2}
+\frac{1}{N}\Vars{g(\cdot\vert\wt{\vtheta}^s)}
\nonumber 
+\frac{\Gamma\VARIS}{B}
\]
\end{restatable}
\begin{proof}
For ease of notation denote $\omega(\tau_i) := \omega(\tau_i|\vtheta^{s+1}_t, \wt{\vtheta}^{s})$. Then,
\begingroup
\allowdisplaybreaks
	\begin{align}
            \mathbb{E}_{t|s}\big[&\norm[2]{\gradBlack{\vtheta_t^{s+1}}}^2\big]
	= \Ets{\norm[]{\gradApp{\wt{\vtheta}^s}{N}
			+\frac{1}{B}\sum_{i=0}^{B-1} g(\tau_i\vert\vtheta_t^{s+1}) 
			-\frac{1}{B}\sum_{i=0}^{B-1}
    \omega(\tau_i)g(\tau_i\vert\wt{\vtheta}^s)}^2} \nonumber\\
	&= \mathbb{E}_{t\vert s}\left[\left\|\gradApp{\wt{\vtheta}^s}{N}
			+\frac{1}{B}\sum_{i=0}^{B-1}\left( 
			g(\tau_i\vert\vtheta_t^{s+1}) -
			\omega(\tau_i)g(\tau_i\vert\wt{\vtheta}^s)\right)
            \pm \gradJ{\vtheta_t^{s+1}} \pm \gradJ{\wt{\vtheta}^s} 
    \right\|^2\right] \nonumber\\
	&\leq \Ets[t-1]{\norm[]{\gradJ{\vtheta_t^{s+1}}}^2}
	+\Es{\norm[]{\gradApp{\wt{\vtheta}^s}{N} - \Es{\gradApp{\wt{\vtheta}^s}{N}}}^2} \nonumber\\
	&\qquad+ 
	\mathbb{E}_{t\vert s}\left[\left\|
		\frac{1}{B}\sum_{i=0}^{B-1}\left(
		g(\tau_i\vert\vtheta_t^{s+1}) -
			\omega(\tau_i)g(\tau_i\vert\wt{\vtheta}^s)\right)
		- \Ets{
			\frac{1}{B}\sum_{i=0}^{B-1}\left(
			g(\tau_i\vert\vtheta_t^{s+1}) -
				\omega(\tau_i)g(\tau_i\vert\wt{\vtheta}^s)\right)}\right\|^2\right] 
	\nonumber\\
	&= \Ets[t-1]{\norm[]{\gradJ{\vtheta_t^{s+1}}}^2}
	+\Vars{\gradApp{\wt{\vtheta}^s}{N}} \nonumber\\
	&\qquad+ 
	\mathbb{E}_{t\vert s}\left[\left\|
	\frac{1}{B}\sum_{i=0}^{B-1}\left(
	g(\tau_i\vert\vtheta_t^{s+1}) -
	\omega(\tau_i)g(\tau_i\vert\wt{\vtheta}^s)\right)
	- \Ets{
		\frac{1}{B}\sum_{i=0}^{B-1}\left(
		g(\tau_i\vert\vtheta_t^{s+1}) -
		\omega(\tau_i)g(\tau_i\vert\wt{\vtheta}^s)\right)}\right\|^2\right] 
	\label{neweq:2}\\
	&= \Ets[t-1]{\norm[]{\gradJ{\vtheta_t^{s+1}}}^2} 
	+\frac{1}{N}\Vars{g(\cdot\vert\wt{\vtheta}^s)}
	\nonumber\\
	&\qquad+ 
		\mathbb{E}_{t\vert s}\left[\left\|
		\frac{1}{B}\sum_{i=0}^{B-1}\left(
		g(\tau_i\vert\vtheta_t^{s+1}) -
		\omega(\tau_i)g(\tau_i\vert\wt{\vtheta}^s)\right)
		- \Ets{
			\frac{1}{B}\sum_{i=0}^{B-1}\left(
			g(\tau_i\vert\vtheta_t^{s+1}) -
			\omega(\tau_i)g(\tau_i\vert\wt{\vtheta}^s)\right)}\right\|^2\right] 
		\label{eq:1}\\%
	&\leq \Ets[t-1]{\norm[2]{\gradJ{\vtheta_t^{s+1}}}^2} 
	+\frac{1}{N}\Vars{g(\cdot\vert\wt{\vtheta}^s)} 
    +\Ets{\norm[2]{
			\frac{1}{B}\sum_{i=0}^{B-1}\left(
			g(\tau_i\vert\vtheta_t^{s+1}) -
			\omega(\tau_i)g(\tau_i\vert\wt{\vtheta}^s)\right)}^2} \label{eq:2}\\
	&\leq \Ets[t-1]{\norm[2]{\gradJ{\vtheta_t^{s+1}}}^2} 
	+\frac{1}{N}\Vars{g(\cdot\vert\wt{\vtheta}^s)}
			+ \frac{1}{B^2}\sum_{i=0}^{B-1}
			\Ets{\norm[2]{
			g(\tau_i\vert\vtheta_t^{s+1}) -
    \omega(\tau_i)g(\tau_i\vert\wt{\vtheta}^s) \pm g(\tau_i|\wt{\vtheta}^s) }^2} \nonumber\\
	%
	%
	&\leq \Ets[t-1]{\norm[]{\gradJ{\vtheta_t^{s+1}}}^2} 
	+\frac{1}{N}\Vars{g(\cdot\vert\wt{\vtheta}^s)}
	\nonumber\\
	&\qquad+
			\frac{1}{B^2}\sum_{i=0}^{B-1}
			\Ets{\norm[]{g(\tau_i\vert\vtheta_t^{s+1})
			-g(\tau_i\vert\wt{\vtheta}^s)}^2} 
			+\frac{1}{B^2}\sum_{i=0}^{B-1}
			\Ets{\norm[]{g(\tau_i\vert\wt{\vtheta}^s) 
			-\omega(\tau_i)g(\tau_i\vert\wt{\vtheta}^s)}^2} \nonumber\\
	&\leq \Ets[t-1]{\norm[]{\gradJ{\vtheta_t^{s+1}}}^2} 
	+\frac{1}{N}\Vars{g(\cdot\vert\wt{\vtheta}^s)}
	\nonumber\\
	&\qquad
    +\frac{L_g^2}{B}\Ets[t-1]{\norm[]{\vtheta_t^{s+1}-\wt{\vtheta}^s}^2}
    +
		\frac{1}{B^2}\sum_{i=0}^{B-1}
		\Ets{\norm[]{(1 
			-\omega(\tau_i))g(\tau_i\vert\wt{\vtheta}^s)}^2} \label{eq:3}\\
% ...
	&\leq \Ets[t-1]{\norm[]{\gradJ{\vtheta_t^{s+1}}}^2} 
	+\frac{1}{N}\Vars{g(\cdot\vert\wt{\vtheta}^s)}
	\nonumber\\
	&\qquad+\frac{L_g^2}{B}\Ets[t-1]{\norm[]{\vtheta_t^{s+1}-\wt{\vtheta}^s}^2}
	+\Gamma\frac{1}{B^2}\sum_{i=0}^{B-1}\Ets{(\omega(\tau_i)-1)^2} \label{eq:4}\\
	&= \Ets[t-1]{\norm[]{\gradJ{\vtheta_t^{s+1}}}^2} 
	+\frac{1}{N}\Vars{g(\cdot\vert\wt{\vtheta}^s)}
    +\frac{L_g^2}{B}\Ets[t-1]{\norm[]{\vtheta_t^{s+1}-\wt{\vtheta}^s}^2}
	+\frac{\Gamma}{B^2}\sum_{i=0}^{B-1}\Varts{\omega(\tau_i)} \nonumber\\
	&\leq \Ets[t-1]{\norm[]{\gradJ{\vtheta_t^{s+1}}}^2} 
	+\frac{1}{N}\Vars{g(\cdot\vert\wt{\vtheta}^s)}
    +\frac{L_g^2}{B}\Ets[t-1]{\norm[]{\vtheta_t^{s+1}-\wt{\vtheta}^s}^2}
	+\frac{\Gamma\VARIS}{B}, \label{eq:5}
\end{align}
\endgroup
where (\ref{neweq:2}) is from (\ref{neweq:1}), (\ref{eq:1}) is from the definition of $\gradApp{\vtheta}{N}$, (\ref{eq:2}) is from Lemma \ref{lemma:varineq}, (\ref{eq:3}) is from Lemma \ref{lemma:gsmooth}, 
(\ref{eq:4}) is from Lemma \ref{lemma:gbound}, and (\ref{eq:5}) is from Assumption \ref{ass:M2}.
\end{proof}

\begin{restatable}[]{lemma}{auxzero}\label{lemma:aux0}
Under Assumption \ref{ass:bounded_score}, for any function $\varphi(\vtheta_t^{s+1})$ which is deterministic for a fixed $\vtheta_t^{s+1}$:
\begin{align*}
\left|\Ets[t]{\dotprod{\gradBlack{\vtheta_t^{s+1}}}{\varphi(\vtheta_t^{s+1})}}
-\Ets{\dotprod{\gradJ{\vtheta_t^{s+1}}}{\varphi(\vtheta_t^{s+1})}}
\right|
\leq
\frac{1}{2N}\Vars{g(\cdot\vert\wt{\vtheta}^s)} +\frac{1}{2}\Ets[t-1]{\norm[]{\varphi(\vtheta_t^{s+1})}^2}
\end{align*}
\end{restatable}
\begin{proof}
\begin{align}
	\Ets{\dotprod{\gradBlack{\vtheta_t^{s+1}}}{\varphi(\vtheta_t^{s+1})}}
	&=
	\Ets{\dotprod{\gradIdeal{\vtheta_t^{s+1}}}{\varphi(\vtheta_t^{s+1})}} +
	\Ets[t-1]{\dotprod{e_s}{\varphi(\vtheta_t^{s+1})}} \label{eq:6}\\
	&=
	\Ets{\dotprod{\gradJ{\vtheta_t^{s+1}}}{\varphi(\vtheta_t^{s+1})}} +
	\Ets[t-1]{\dotprod{e_s}{\varphi(\vtheta_t^{s+1})}} \label{eq:7}\\
	&=
	\Ets{\dotprod{\gradJ{\vtheta_t^{s+1}}}{\varphi(\vtheta_t^{s+1})}} \nonumber\\
	&\qquad+
	\dotprod{\Ets{e_s}}{\Ets{\varphi(\vtheta_t^{s+1})}}
	+\Covts[t-1]{\gradApp{\wt{\vtheta}^s}{N}}{\varphi(\vtheta_t^{s+1})}  \label{eq:new1}\\
	&= 
	\Ets{\dotprod{\gradJ{\vtheta_t^{s+1}}}{\varphi(\vtheta_t^{s+1})}} \nonumber\\
	&\qquad+
	\Covts[t-1]{\gradApp{\wt{\vtheta}^s}{N}}{\varphi(\vtheta_t^{s+1})} \label{eq:8}
\end{align}
where~\eqref{eq:6} is from Definition~\ref{def:ideal};~\eqref{eq:7} is from the fact that $\gradIdeal{\vtheta_t^{s+1}}$ is both unbiased and independent from $\varphi(\vtheta_t^{s+1})$ \wrt the sampling at time $t$ alone, which is not true for $\gradBlack{\vtheta_t^{s+1}}$;~\eqref{eq:new1} is from the fact that $\gradJ{\wt{\vtheta}^s}$ is constant \wrt $\Vars{\cdot}$;~\eqref{eq:8} is from $\Ets{e_s}=0$.
Hence:
\begin{align}
	\left|\Ets{\dotprod{\gradBlack{\vtheta_t^{s+1}}}{\varphi(\vtheta_t^{s+1})}}
	\right.&-\left.\Ets{\dotprod{\gradJ{\vtheta_t^{s+1}}}{\varphi(\vtheta_t^{s+1})}}\right| 
	=
	\left|\Covts[t-1]{\gradApp{\wt{\vtheta}^s}{N}}{\varphi(\vtheta_t^{s+1})}\right|  
	\nonumber\\
	&\leq
    \sqrt{\Vars{\gradApp{\wt{\vtheta}^s}{N}}} \cdot \sqrt{\Varts[t-1]{\varphi(\vtheta_t^{s+1})}} \label{eq:9a}\\
	&\leq	
	\frac{1}{2}\Vars{\gradApp{\wt{\vtheta}^s}{N}} +\frac{1}{2}\Varts[t-1]{\varphi(\vtheta_t^{s+1})}\label{eq:9}\\
	&=
	\frac{1}{2N}\Vars{g(\cdot\vert\wt{\vtheta}^s)} +\frac{1}{2}\Varts[t-1]{\varphi(\vtheta_t^{s+1})} \label{eq:10}\\
	&\leq
	\frac{1}{2N}\Vars{g(\cdot\vert\wt{\vtheta}^s)} +\frac{1}{2}\Ets[t-1]{\norm[]{\varphi(\vtheta_t^{s+1})}^2},
	\label{neweq:3}
\end{align}
where~\eqref{eq:9a} comes from Cauchy-Schwarz inequality
\[
        |\Cov(X,Y)| = |\EVV[]{(X-\mu_X)\transpose{(Y-\mu_Y)}}| \leq \EVV[]{(X-\mu_X)^2}^{1/2}\EVV[]{(Y-\mu_Y)^2}^{1/2} = \sqrt{\Var(X) \Var(Y)},
\]
\eqref{eq:9} is from Young's inequality,~\eqref{eq:10} is from the definition of $\gradApp{\vtheta}{N}$, and \eqref{neweq:3} is from Lemma \ref{lemma:varineq}.
\end{proof}

\begin{restatable}[]{lemma}{auxone}\label{lemma:aux1}
Under Assumptions \ref{ass:bounded_score} ans \ref{ass:M2}, the expected squared norm of the true gradient $\gradJ{\vtheta_t^{s+1}}$, for appropriate choices of $\alpha_t\geq0$ and $\beta_t>0$, can be bounded as follows:
\[
	\Ets[t-1]{\norm[]{\gradJ{\vtheta_t^{s+1}}}^2} \leq
	\frac{R_{t+1}^{s+1} - R_t^{s+1}}{\Psi_t} + \frac{d_tV}{N\Psi_t}
	+\frac{f_tW}{B\Psi_t},
\]
	where
\begin{align*}
	&R_t^{s+1}\coloneqq \Ets[t-1]{J(\vtheta_t^{s+1}) - c_t\norm[]{\vtheta_t^{s+1}-\wt{\vtheta}^s}^2}, \\
	&c_{m} = 0, \\
	&c_t = c_{t+1}\left(1+\alpha_t\beta_t+\alpha_t+\frac{\alpha_t^2L^2}{B}\right)+\frac{\alpha_t^2L^3}{2B}, \\
	&\Psi_t = \alpha_t\left(\frac{1}{2}-\frac{c_{t+1}}{\beta_t}-\frac{\alpha_tL}{2}-\alpha_tc_{t+1}\right), \\
	&d_t = \frac{\alpha_t}{2}\left(1+2c_{t+1}+\alpha_tL+2\alpha_tc_{t+1}\right), \\
	&f_t = \alpha_t^2\frac{\Gamma(L+2c_{t+1})}{2},
\end{align*}
where $L=\max\left\{L_J,L_g\right\}$, \ie the greater of the Lipschitz constants from Lemmas \ref{lemma:lsmooth} and \ref{lemma:gsmooth}.

In particular, the following constraints on $\alpha_t$ and $\beta_t$ are sufficient:
\begin{align*}
&0\leq\alpha_t < \frac{1-\nicefrac{2c_{t+1}}{\beta_t}}{L+2c_{t+1}} \\
&\beta_t > 2c_{t+1}.
\end{align*}
\end{restatable}
\begin{proof}
	We have:
	\begin{align}
	\Ets{J(\vtheta_{t+1}^{s+1})} 
	&\geq \Ets{J(\vtheta_t^{s+1})+\dotprod{\gradJ{\vtheta_t^{s+1}}}{\vtheta_{t+1}^{s+1}-\vtheta_t^{s+1}} - \frac{L}{2}\norm[]{\vtheta_{t+1}^{s+1}-\vtheta_t^{s+1}}^2} \label{eq:11}\\
	&= \Ets{J(\vtheta_t^{s+1})+\alpha_t\dotprod{\gradJ{\vtheta_t^{s+1}}}{\gradBlack{\vtheta_t^{s+1}}} - \frac{\alpha_t^2L}{2}\norm[]{\gradBlack{\vtheta_t^{s+1}}}^2} \label{eq:12}\\
	&\geq
	\Ets{J(\vtheta_t^{s+1})+\alpha_t\norm[]{\gradJ{\vtheta_t^{s+1}}}^2 - \frac{\alpha_t^2L}{2}\norm[]{\gradBlack{\vtheta_t^{s+1}}}^2} \nonumber\\
	&\qquad-
	\frac{\alpha_t}{2N}\Vars{g(\cdot\vert\wt{\vtheta}^s)} -\frac{\alpha_t}{2}\Ets[t-1]{\norm[]{\gradJ{\vtheta_t^{s+1}}}^2}, \label{eq:13}
\end{align}
where~\eqref{eq:11} is from the L-smoothness of $J(\vtheta)$ \cite{nesterov2013introductory} and \eqref{eq:12} is from the SVRPG update.
Inequality~\ref{eq:13} follows from Lemma~\ref{lemma:aux0} by noticing that $\nabla J(\vtheta^{s+1}_t)$ is a deterministic function given $\vtheta^{s+1}_t$. As a consequence, we can directly apply Lemma~\ref{lemma:aux0} with $\vphi(\vtheta_t^{s+1})\coloneqq\gradJ{\vtheta_t^{s+1}}$.

Next, we have:
\begingroup
\allowdisplaybreaks
\begin{align}
        \mathbb{E}_{t|s}\bigg[ & \norm[]{\vtheta_{t+1}^{s+1}-\wt{\vtheta}^s}^2 \bigg]
= \Ets{\norm[]{\vtheta_{t+1}^{s+1}- \vtheta_t^{s+1} + \vtheta_t^{s+1}-\wt{\vtheta}^s}^2} \nonumber\\
&=\Ets{\norm[]{\vtheta_{t+1}^{s+1}-\vtheta_{t}^{s+1}}^2+\norm[]{\vtheta_t^{s+1}-\wt{\vtheta}^s}^2+2\dotprod{\vtheta_{t+1}^{s+1}-\vtheta_{t}^{s+1}}{\vtheta_t^{s+1}-\wt{\vtheta}^s}} \label{eq:14a} \\
&= \Ets{\alpha_t^2\norm[]{\gradBlack{\vtheta_t^{s+1}}}^2+\norm[]{\vtheta_t^{s+1}-\wt{\vtheta}^s}^2+2\alpha_t\dotprod{\gradBlack{\vtheta_t^{s+1}}}{\vtheta_t^{s+1}-\wt{\vtheta}^s}} \label{eq:14}\\
&\leq \Ets{\alpha_t^2\norm[]{\gradBlack{\vtheta_t^{s+1}}}^2+\norm[]{\vtheta_t^{s+1}-\wt{\vtheta}^s}^2+2\alpha_t\dotprod{\gradJ{\vtheta_t^{s+1}}}{\vtheta_t^{s+1}-\wt{\vtheta}^s}} \nonumber\\ 
&\qquad+
\frac{\alpha_t}{N}\Vars{g(\cdot\vert\wt{\vtheta}^s)} +\alpha_t\Ets[t-1]{\norm[]{\vtheta_t^{s+1}-\wt{\vtheta}^s}^2} \label{eq:15}\\
&\leq \Ets{\alpha_t^2\norm[]{\gradBlack{\vtheta_t^{s+1}}}^2+\norm[]{\vtheta_t^{s+1}-\wt{\vtheta}^s}^2}
+2\alpha_t\Ets[t-1]{\left|\dotprod{\gradJ{\vtheta_t^{s+1}}}{\vtheta_t^{s+1}-\wt{\vtheta}^s}\right|} \nonumber\\ 
&\qquad+
\frac{\alpha_t}{N}\Vars{g(\cdot\vert\wt{\vtheta}^s)} +\alpha_t\Ets[t-1]{\norm[]{\vtheta_t^{s+1}-\wt{\vtheta}^s}^2} \nonumber\\
&\leq \Ets{\alpha_t^2\norm[]{\gradBlack{\vtheta_t^{s+1}}}^2+\norm[]{\vtheta_t^{s+1}-\wt{\vtheta}^s}^2}
+2\alpha_t\Ets[t-1]{\norm[]{\gradJ{\vtheta_t^{s+1}}}\norm[]{\vtheta_t^{s+1}-\wt{\vtheta}^s}} \nonumber\\ 
&\qquad+
\frac{\alpha_t}{N}\Vars{g(\cdot\vert\wt{\vtheta}^s)} +\alpha_t\Ets[t-1]{\norm[]{\vtheta_t^{s+1}-\wt{\vtheta}^s}^2} \nonumber\\
&\leq \Ets{\alpha_t^2\norm[]{\gradBlack{\vtheta_t^{s+1}}}^2+\norm[]{\vtheta_t^{s+1}-\wt{\vtheta}^s}^2}
+2\alpha_t\Ets[t-1]{\frac{1}{2\beta_t}\norm[]{\gradJ{\vtheta_t^{s+1}}}^2+\frac{\beta_t}{2}\norm[]{\vtheta_t^{s+1}-\wt{\vtheta}^s}^2} \label{eq:16a}\\ 
&\qquad
+\frac{\alpha_t}{N}\Vars{g(\cdot\vert\wt{\vtheta}^s)} +\alpha_t\Ets[t-1]{\norm[]{\vtheta_t^{s+1}-\wt{\vtheta}^s}^2}, \label{eq:16}
\end{align}
\endgroup
where~\eqref{eq:14a} is obtained using the triangular inequality,~\eqref{eq:14} is from the SVRPG update,~\eqref{eq:15} is from Lemma~\ref{lemma:aux0} with $\vphi(\vtheta_t^{s+1})\coloneqq\vtheta_t^{s+1}-\tilde{\vtheta}^s$,~\eqref{eq:16a} is from Cauchy-Schwarz inequality, and~\eqref{eq:16} is from Young's inequality in the `Peter-Paul' variant.
Let us consider the following function:
\begin{equation}\label{E:lyapunov.function}
	R_{t+1}^{s+1} \coloneqq \Ets{J(\vtheta_{t+1}^{s+1}) - c_{t+1}\norm[]{\vtheta_{t+1}^{s+1}-\tilde{\vtheta}^s}^2}. 
\end{equation}
The objective is now to provide a lower bound to it.
\begingroup
\allowdisplaybreaks
\begin{align}
	R_{t+1}^{s+1} 
	&\geq	\Ets{J(\vtheta_t^{s+1}) - \frac{\alpha_t^2L}{2}\norm[]{\gradBlack{\vtheta_t^{s+1}}}^2}
    + \EVV[t-1|s]{\frac{\alpha_t}{2}\norm[]{\gradJ{\vtheta_t^{s+1}}}^2}
    \nonumber\\
	&\qquad-
	\frac{\alpha_t}{2N}\Vars{g(\cdot\vert\tilde{\vtheta}^s)}
	-c_{t+1}\Ets{\norm[]{\vtheta_{t+1}^{s+1}-\tilde{\vtheta}^s}^2} \label{eq:17}\\
	&\geq \Ets{J(\vtheta_t^{s+1}) - \frac{\alpha_t^2L}{2}\norm[]{\gradBlack{\vtheta_t^{s+1}}}^2} 
	+ \Ets[t-1]{\frac{\alpha_t}{2}\norm[]{\gradJ{\vtheta_t^{s+1}}}^2 }
	-\frac{\alpha_t}{2N}\Vars{g(\cdot\vert\tilde{\vtheta}^s)} \nonumber\\
	&\qquad -c_{t+1}\Ets{\alpha_t^2\norm[]{\gradBlack{\vtheta_t^{s+1}}}^2+\norm[]{\vtheta_t^{s+1}-\tilde{\vtheta}^s}^2}
	\nonumber\\
	&\qquad-2c_{t+1}\alpha_t\Ets[t-1]{\frac{1}{2\beta_t}\norm[]{\gradJ{\vtheta_t^{s+1}}}^2+\frac{\beta_t}{2}\norm[]{\vtheta_t^{s+1}-\tilde{\vtheta}^s}^2} \nonumber\\ 
	&\qquad
	-c_{t+1}\frac{\alpha_t}{N}\Vars{g(\cdot\vert\tilde{\vtheta}^s)} -c_{t+1}\alpha_t\Ets[t-1]{\norm[]{\vtheta_t^{s+1}-\tilde{\vtheta}^s}^2} \label{eq:18}\\
	&= \Ets[t-1]{J(\vtheta_t^{s+1})} - c_{t+1}\left(1+\alpha_t\beta_t+\alpha_t\right)\Ets[t-1]{\norm[]{\vtheta_{t}^{s+1}-\tilde{\vtheta}}^2} \nonumber\\
	&\qquad-\alpha_t^2\left(\frac{L}{2}+c_{t+1}\right)\Ets{\norm[]{\gradBlack{\vtheta_t^{s+1}}}^2}
	+\frac{\alpha_t}{2}\left(1-\frac{2c_{t+1}}{\beta_t}\right)\Ets[t-1]{\norm[]{\gradJ{\vtheta_t^{s+1}}}^2} \nonumber\\
	&\qquad-\frac{\alpha_t}{2N}\left(1+2c_{t+1}\right)\Vars{g(\cdot\vert\tilde{\vtheta}^s)} \nonumber\\
	&\geq  \Ets[t-1]{J(\vtheta_t^{s+1})} - c_{t+1}\left(1+\alpha_t\beta_t+\alpha_t\right)\Ets[t-1]{\norm[]{\vtheta_{t}^{s+1}-\tilde{\vtheta}}^2} \nonumber\\
	&\qquad
	-\alpha_t^2\left(\frac{L}{2}+c_{t+1}\right)\left(\Ets[t-1]{\norm[]{\gradJ{\vtheta_t^{s+1}}}^2} 
	+\frac{1}{N}\Vars{g(\cdot\vert\tilde{\vtheta}^s)}
	\right.\nonumber\\
	&\left.\qquad+\frac{L^2}{B}\Ets[t-1]{\norm[]{\vtheta_t^{s+1}-\tilde{\vtheta}^s}^2}
	+\frac{\Gamma W}{B}\right)
	+\frac{\alpha_t}{2}\left(1-2\frac{c_{t+1}}{\beta_t}\right)\Ets[t-1]{\norm[]{\gradJ{\vtheta_t^{s+1}}}^2} \nonumber\\
	&\qquad-\frac{\alpha_t}{2N}\left(1+2c_{t+1}\right)\Vars{g(\cdot\vert\tilde{\vtheta}^s)} \label{eq:19}\\
	& = \Ets[t-1]{J(\vtheta_t^{s+1}) - \left(c_{t+1}\left(1+\alpha_t\beta_t+\alpha_t+\frac{\alpha_t^2L^2}{B}\right)+\frac{\alpha_t^2L^3}{2B}\right)\norm[]{\vtheta_{t}^{s+1}-\tilde{\vtheta}}^2} \nonumber\\
	&\qquad
	+\alpha_t\left(\frac{1}{2}-\frac{c_{t+1}}{\beta_t}-\frac{\alpha_tL}{2}-\alpha_tc_{t+1}\right)\Ets[t-1]{\norm[]{\gradJ{\vtheta_t^{s+1}}}^2} \nonumber\\
	&\qquad-\frac{\alpha_t}{2N}\left(1+2c_{t+1}+\alpha_tL+2\alpha_tc_{t+1}\right)\Vars{g(\cdot\vert\tilde{\vtheta}^s)} 
    -\alpha_t^2\frac{(L+2c_{t+1})\Gamma\VARIS}{2B} \nonumber\\
	&= R_t^{s+1}
	+\Psi_t\Ets[t-1]{\norm[]{\gradJ{\vtheta_t^{s+1}}}^2}
	-\frac{d_t}{N}\Vars{g(\cdot\vert\tilde{\vtheta}^s)}
	-\frac{f_t}{B}\VARIS,\nonumber\\
	&\geq R_t^{s+1}
	+\Psi_t\Ets[t-1]{\norm[]{\gradJ{\vtheta_t^{s+1}}}^2}
	-\frac{d_t}{N}\VARRF
	-\frac{f_t}{B}\VARIS, \label{eq:20}
\end{align}
\endgroup
where~\eqref{eq:17} is from~\eqref{eq:13} noticing that $\EVV[t|s]{\norm[]{\nabla J(\vtheta^{s+1}_t)}^2} = \EVV[t-1|s]{\norm[]{\nabla J(\vtheta^{s+1}_t)}^2}$,~\eqref{eq:18} is from~\eqref{eq:16},~\eqref{eq:19} is from Lemma~\ref{lemma:aux2}, and (\ref{eq:20}) is from Assumption \ref{ass:REINFORCE}.
To complete the proof, besides rearranging terms, we have to ensure that $\Psi_t>0$ for each $t$. This gives the constraints on $\alpha_t$ and $\beta_t$.
\end{proof}

\subsection*{Main theorem}
We finally provide the proof of the convergence theorem:

\convergence*
\begin{proof}
We prove the theorem for the following values of the constants:
\begin{align*}
& \psi \coloneqq \min_t\{\Psi_t\}, 
& \zeta \coloneqq \frac{\max_t\{d_t\}V}{\psi}, 
&& \xi \coloneqq \frac{\max_t\{f_t\}W}{\psi},
\end{align*}
where $\Psi$, $d_t$ and $f_t$ are defined in Lemma~\ref{lemma:aux1}.
Starting from Lemma \ref{lemma:aux1}, summing over iterations of an epoch $s$ and using telescopic sum we obtain
\begin{align*}
\sum_{t=0}^{m-1}\Ets{\norm[]{\gradJ{(\vtheta_t^{s+1})}^2}}&\leq
 \frac{\sum_{t=0}^{m-1}\left(R_{t+1}^{s+1} - R_t^{s+1}\right)}{\psi} + \frac{m\zeta}{N} + \frac{m\xi}{B} \nonumber\\
 & = \frac{R^{s+1}_m - R^{s+1}_0}{\psi}  + \frac{m\zeta}{N} + \frac{m\xi}{B}
\end{align*}
By using the definition of $R^s_t$ in~\eqref{E:lyapunov.function}, the fact that $c_m = 0$ and $\vtheta^{s+1}_0 = \wt{\vtheta}^s = \vtheta^s_m$, we can state that:
\begin{align*}
        R^{s+1}_m - R^{s+1}_0 
        &= \EVV[m|s]{J(\vtheta^{s+1}_{m})- c_m \norm[]{\vtheta^{s+1}_{m} - \wt{\vtheta}^s}^2 } - \EVV[0|s]{J(\vtheta^{s+1}_{0}) - c_0 \norm[]{\vtheta^{s+1}_0 - \wt{\vtheta}^s}^2}\\
        &= \EVV[m|s]{J(\vtheta^{s+1}_{m})} - \EVV[0|s]{J(\wt{\vtheta}^s)}
        = \EVV[m|s]{J(\wt{\vtheta}^{s+1}) - J(\wt{\vtheta}^s)}
\end{align*}
Next, summing over epochs:
\begin{align}
\sum_{s=0}^{S-1}\sum_{t=0}^{m-1}\Ets{\norm[]{\gradJ{(\vtheta_t^{s+1})}^2}}&\leq
\frac{\sum_{s=0}^{S-1}\Ets[m]{J(\tilde{\vtheta}^{s+1}) - J(\tilde{\vtheta}^{s})}}{\psi} + \frac{T\zeta}{N} + \frac{T\xi}{B} \nonumber\\
&\leq
\frac{\EVV[]{J(\tilde{\vtheta}^{S}) - J(\tilde{\vtheta}^{0})}}{\psi} + \frac{T\zeta}{N} + \frac{T\xi}{B}
 \label{eq:23}\\
&\leq
\frac{J(\vtheta^*) - J(\vtheta^0)}{\psi} + \frac{T\zeta}{N} + \frac{T\xi}{B}, \label{eq:24}
\end{align}
where the expectation in~\eqref{eq:23} is \wrt all the trajectories sampled in a run of Algorithm~\ref{alg:svrpg} and~\eqref{eq:24} is from the definition of $\vtheta^*$ (\ie the policy performance maximizer).
Finally, we consider the expectation \wrt all sources of randomness, including the uniform sampling of the output parameter:
\begin{align*}
\EVV[]{\norm[]{\gradJ{(\vtheta_t^{s+1})}}^2} 
&=\frac{1}{T}\sum_{s=0}^{S-1}\sum_{t=0}^{m-1}\Ets{\norm[]{\gradJ{(\vtheta_t^{s+1})}}^2} 
\leq
\frac{J(\vtheta^*) - J(\vtheta^0)}{\psi T} + \frac{\zeta}{N} + \frac{\xi}{B}.
\end{align*}

\end{proof}

\section{Applicability to Gaussian Policies}\label{app:gauss}
We provide more details on the applicability of Theorem \ref{theo:convergence} on the case of Gaussian policies. We start from the case of one-dimensional bounded action space $\mathcal{A}\subset\mathbb{R}$, linear mean $\mu(s) = \vtheta^T\vphi(s)$ and fixed standard deviation $\sigma$:
\[
	\pi_{\vtheta}(a\vert s) = \frac{1}{\sqrt{2\pi}\sigma}\exp\left\{
		-\frac{(\vtheta^T\vphi(s) - a)^2}{2\sigma^2}\right\},
\]
where $\vphi(s)\leq M_{\phi}$ is a bounded feature vector, and we see under which conditions the three assumptions of Section \ref{sec:conv} hold.

\boundedscore*
For the Gaussian policy defined above, it's easy to show that:
\begin{align*}
	&\nabla_{\theta_i}\log\pi_{\vtheta}(\tau) =  \phi_i(s)\frac{a-\vtheta^T\phi(s)}{\sigma^2},\\
	&\frac{\partial^2}{\partial\theta_i\partial\theta_j}\log\pi_{\vtheta}(\tau) = \frac{\phi_i(s)\phi_j(s)}{\sigma^2}.
\end{align*}
Hence, Assumption \ref{ass:bounded_score} is automatically satisfied \footnote{This relies on the fact that $\vtheta^T\phi(s)$ lies in bounded $\Aspace$. In practice, this is usually enforced by clipping the action selected by $\pi_{\vtheta}$. A more rigorous way would be to employ the truncated Gaussian distribution.} by taking $\GRADLOG = \frac{M_{\phi}|\Aspace|}{\sigma^2}$ and $\HESSLOG = \frac{M_{\phi}^2}{\sigma^2}$.
\par
\varreinforce*
As mentioned, \cite{pirotta2013adaptive} provides a bound on the variance of the REINFORCE estimator, adapted from \cite{zhao2011analysis}, which does not depend on $\vtheta$:
\begin{align*}
\Var\left[\gradApp{\theta_i}{N}\right] \leq \frac{R^2M_{\phi}^2H(1-\gamma^H)^2}{N\sigma^2(1-\gamma)^2}.
\end{align*}
The same authors provide a similar bound for G(PO)MDP.

\varweights*
It is noted in \cite{cortes2010learning} that, for any two Gaussian distributions $\mathcal{N}(\mu_1,\sigma_1)$ and $\mathcal{N}(\mu_2,\sigma_2)$, the variance of the importance weights from the latter to the former is bounded whenever $\sigma_2 > \frac{\sqrt{2}}{2}\sigma_1$. This is automatically satisfied by our fixed-variance Gaussian policies, since $\sigma_2=\sigma_1=\sigma$.

We now briefly examine some generalizations of the simple Gaussian policy defined above that can be found in applications:

\paragraph{Multi-dimensional actions.}
When actions are multi-dimensional, factored Gaussian policies are typically employed, so the results extend trivially from the one-dimensional case. Actual multi-variate Gaussian distributions would require more calculations, but we do not expect substantially different results.

\paragraph{Non-linear mean.}
In complex continuous tasks, $\mu(s)$ often represents a deep neural network, or multi-layer perceptron, where $\vtheta$ are the weights of the network. The analysis of first and second order log-derivatives in such a scenario is beyond the scope of this paper.

\paragraph{Adaptive variance.}
It is a common practice to learn also the variance of the policy in order to adapt the degree of exploration. The variance (or diagonal covariance matrix in the multi-dimensional case) can be learned as a separate parameter or be state-dependent like the mean. In any case, adaptive variance must be carefully employed since it can clearly undermine all the three assumptions of Theorem \ref{theo:convergence}.

\section{Practical SVRPG Versions}\label{app:practicalsvrpg}
We provide more details on the practical variants of SVRPG.

\subsection{Adpative Step Size}
\begin{algorithm}[h]
	\begin{algorithmic}
		\STATE \textbf{Input:} A gradient estimate $g_t$ and parameters $\beta_1$, $\beta_2$, $\epsilon$ and $\alpha$.
		\STATE $\kappa_t = \beta_1 \kappa_{t-1} + (1 - \beta_1) g_t$
		\STATE $\nu_t = \beta_2 \nu_{t-1} + (1 - \beta_2) g_t \circ g_t$ ($\circ$ is the Hadamard (component-wise) product)
		\STATE $\hat{\kappa}_t = \dfrac{\kappa_t}{1 - \beta^t_1}$
		\STATE $\hat{\nu}_t = \dfrac{\nu_t}{1 - \beta^t_2}$
		\STATE $\Delta(g_t) = \dfrac{\alpha}{\sqrt{\hat{\nu}_t} + \epsilon} \hat{\kappa}_t$
		\STATE \textbf{Return:} The increment $\Delta(g_t)$  of the parameters.
	\end{algorithmic}
	\caption{
		\label{A:adam}
		Adam}
\end{algorithm}
Let us give a deeper insight on the two different learning rate schedules used by our algorithm. We report pseudo-code of the original ADAM \cite{kingma2014adam} in Algorithm \ref{A:adam}. As mentioned, we use two distinct instances of ADAM to manage different sources of variance: one related to the snapshots, and one to the sub-iterations. In this way the ADAM associated to the snapshots takes into account only the history of gradient moments at the snapshots. By using Algorithm \ref{A:adam} as a subroutine $\textbf{ADAM}(g,\alpha,\beta)$, we can explicitly define our gradient updates:

\begin{align*}
\vtheta^{s+1}_1 &= \wt{\vtheta}^s + \textbf{ADAM}\left(\wh{\nabla}_N J(\wt{\vtheta}^s),\beta_1,\beta_2,\alpha\right),\\
\vtheta^{s+1}_{t+1} &= \textbf{ADAM}\Big( 
\blacktriangledown J(\vtheta^{s+1}_t),\beta_1,\beta_2,\frac{\alpha}{2}\Big)
\text{ for $t=1,\dots,m-1$},
\end{align*}

where separate histories are kept for estimated first moments $\kappa_{FG},\kappa_{IS}$ and estimated second moments $\nu_{FG},\nu_{IS}$.
The meta-parameters $\alpha,\beta_1,\beta_2$ are constant and set to default values or with minor manual tuning (see table \ref{table:metaparams}). Note that we double the sub-iterations' learning rate for the snapshot ADAM since we can rely on a larger number of trajectories ($N$ instead of $B$) to control the variance. 

\subsection{Baseline}
The baseline used in the Half Cheetah experiment is the one used in \cite{duan2016benchmarking}. It is a linear state-value function estimator, or critic. 
The (time-varying) feature encoding for the linear baseline is:
\begin{align*}
\vphi(s,t)=[s, s \odot s, 0.01t, (0.01t)^2, (0.01t)^3,1],
\end{align*}
where $s\in\mathbb{R}^d$ is the state vector and $\odot$ is the element-wise product. The baseline is then:
\[
b(s_t,a_t) = \mathbf{\lambda}^T\vphi(s_t,t).
\]
The baseline is fitted from scratch at each policy gradient iteration, with least squares, to match state-value function $V^{\pi}(s)$.
When used with SVRPG, the critic parameter $\mathbf{\lambda}$ is updated only at the snapshot.


\section{Experimental Details}\label{app:exp}
We describe the RL tasks of Section \ref{sec:exp} in more detail:
\begin{enumerate}
	\item \emph{Cart-Pole Balancing} : an inverted pendulum mounted on a cart must be kept standing by moving the cart backward or forward ;4-dimensional state space: cart position x, pole angle $\theta$, cart velocity $\dot{x}$ and pole velocity $\dot{\theta}$; 1-dimensional action space: the horizontal force applied to the cart body. Reward function  is defined as $r(s, a) := 10 - (1 - cos(\theta)) - 10^{-5}\norm[] a^2$. The episodes terminate when $|x|>2.4$ or $|\theta|>0.2$ or the number of time steps T is greater than 100.
    \item \emph{Mujoco Swimmer}: a snake-like robot immersed in a fluid must move forward; 13-dimensional state space: 3 links velocities ($v_x$ and $v_y$ of center of masses) and 2 actuated joints angles. 2-dimensional action space: the two momentums applied on actuated joints.  The reward function is defined as $r(s, a) := v_x - 10^{-4}\norm[2]{a}^2$. The episodes terminate when the number of time steps T is greater than 500.
    \item \emph{Mujoco Half Cheetah}: a planar biped robot must move forward; 20-dimensional state space: 9 links and 6 actuated joints angles. 6-dimensional action space: the 6 momentums applied on actuated joints.  The reward function is defined as $r(s, a) := v_x - 0.05\norm[2]{a}^2$. The episodes terminate when the number of time steps T is greater than 500.
\end{enumerate}

All the parameters used in the experiments, including neural network architectures, are reported in the following table:

\begin{table}[H]\caption{Parameters used in the experiments of Section \ref{sec:exp}.}\label{table:metaparams}
	\centering
	\begin{tabular}{| l | c  c  c |}
		\hline	
		& Cart-Pole & Swimmer & Half Cheetah \\
		\hline
		NN hidden weights & 8 & 32x32 & 100x50x25 \\
		NN activation & tanh & tanh & tanh \\
		Adam $\alpha$ (SVRPG) & $5\cdotp10^{-2}$ & $10^{-3}$ & $10^{-3}$ \\
		Adam $\alpha$ (GPOMDP) & $10^{-2}$ & $10^{-3}$ & $10^{-2}$ \\
		Adam $\beta_1$ & 0.9 & 0.9 & 0.9 \\
		Adam $\beta_2$ & 0.99 & 0.99 & 0.99 \\ 
		Snapshot batch size $N$ (SVRPG) & 100 & 100 & 100 \\
		Mini-batch size $B$ (SVRPG) & 10 & 10 & 10 \\
		Batch size (GPOMDP) & 10 & 10 & 100 \\
		Max number of sub-iterations & 50 & 20 & 20 \\
		Task horizon& 100 & 500 & 500 \\
		Baseline& No & No & Yes \\
		Discount factor $\gamma$& 0.99 & 0.995 & 0.99 \\
		Total number of trajectories& 10000 & 20000 & 50000 \\
		\hline  
	\end{tabular}
\end{table}
Where not specified, meta-parameters are shared among G(PO)MDP and SVRPG.
Refer to \cite{duan2016benchmarking} for more details about G(PO)MDP (REINFORCE in the paper) on the Half Cheetah task.

\paragraph{Videos:} The supplementary materials include videos showing the behavior of the final SVRPG agents on the three considered tasks. 

\end{document}